\documentclass{article}

\usepackage{microtype}
\usepackage{graphicx}
\usepackage{subfigure}
\usepackage{amsmath}
\usepackage{amsfonts}
\usepackage{color}
\usepackage{tcolorbox}
\usepackage{diagbox}
\usepackage{booktabs} 
\usepackage{mathrsfs}
\usepackage{hyperref}

\newcommand{\RIGHTCOMMENT}[1]{\bgroup\hfill~#1\egroup}
\newtheorem{assumption}{Assumption}

\newtheorem{lemma}{Lemma}
\newtheorem{thm}{Theorem}
\newenvironment{proof}{\paragraph{Proof:}}{\hfill$\square$}

\def \mx {\mathbf{x}}
\def \my {\mathbf{y}}
\def \mz {\mathbf{z}}
\def \mb {\mathbf{b}}
\def \ms {\mathbf{s}}
\def \mm {\mathbf{m}}
\def \mw {\mathbf{w}}
\def \mI {\mathbf{I}}
\def \mzero {\mathbf{0}}
\def \meps {\boldsymbol{\epsilon}}
\def \mv {\mathbf{v}}

\definecolor{peach-orange}{rgb}{1.0, 0.8, 0.6}
\definecolor{piggypink}{rgb}{0.99, 0.87, 0.9}
\definecolor{aurometalsaurus}{rgb}{0.43, 0.5, 0.5}

\usepackage[accepted]{icml2021}

\icmltitlerunning{Deep Generative Learning via Schr\"{o}dinger Bridge}

\begin{document}

\twocolumn[
\icmltitle{Deep Generative Learning via Schr\"{o}dinger Bridge}



\icmlsetsymbol{equal}{*}

\begin{icmlauthorlist}
\icmlauthor{Gefei Wang}{ust}
\icmlauthor{Yuling Jiao}{whu}
\icmlauthor{Qian Xu}{web}
\icmlauthor{Yang Wang}{ust,jointlab}
\icmlauthor{Can Yang}{ust,jointlab}
\end{icmlauthorlist}

\icmlaffiliation{ust}{Department of Mathematics, The Hong Kong University of Science
and Technology, Hong Kong, China}
\icmlaffiliation{whu}{School of Mathematics and Statistics, Wuhan University, Wuhan, China}
\icmlaffiliation{web}{AI Group, WeBank Co., Ltd., Shenzhen, China}
\icmlaffiliation{jointlab}{Guangdong-Hong Kong-Macao Joint Laboratory for Data-Driven Fluid Mechanics and Engineering Applications, The Hong Kong University of Science
and Technology, Hong Kong, China}

\icmlcorrespondingauthor{Yuling Jiao}{yulingjiaomath@whu.edu.cn}
\icmlcorrespondingauthor{Can Yang}{macyang@ust.hk}

\icmlkeywords{Machine Learning, ICML}

\vskip 0.3in
]



\printAffiliationsAndNotice{}  

\begin{abstract}
We propose to learn a generative model via   entropy interpolation with a Schr\"{o}dinger Bridge.
The generative learning task can be formulated as  interpolating between a reference distribution and a target distribution
based on the Kullback-Leibler divergence. At the population level, this entropy interpolation   is  characterized  via an SDE on $[0,1]$ with a time-varying drift term.
 At the sample level, we derive our Schr\"{o}dinger Bridge algorithm by plugging the  drift term estimated  by a deep score estimator and a deep density ratio estimator into the Euler-Maruyama method. Under some mild smoothness assumptions of the target distribution,
we prove the consistency
of both the score estimator and the density ratio estimator, and then establish
the consistency of the proposed Schr\"{o}dinger Bridge approach. Our theoretical results guarantee that the distribution learned by our approach converges to the target distribution.
Experimental results on multimodal synthetic data and benchmark data support our theoretical findings and indicate that the generative model via Schr\"{o}dinger Bridge  is comparable  with state-of-the-art  GANs, suggesting a new formulation of generative learning. We demonstrate its usefulness in image interpolation and image inpainting.
\end{abstract}

\section{Introduction}

Deep generative models have achieved enormous success in learning the underlying high-dimensional data distribution from samples. They have various applications in machine learning, like image-to-image translation \cite{zhu17,Choi_2020_CVPR}, semantic image editing \cite{zhu16, shen2020interpreting} and audio synthesis \citep{oord16, prenger2019waveglow}. Most of existing generative models seek to learn a nonlinear function to transform a simple reference distribution to the target distribution as data generating mechanisms. They can be categorized as either likelihood-based models or implicit generative models.

Likelihood-based models, such as variational auto-encoders (VAEs) \cite{kingma14} and flow-based methods \cite{dinh2014nice}, optimize the negative log-likelihood or its surrogate loss, which is equivalent to minimize the KL-divergence between the target distribution and the generated distribution. Although their ability to learn flexible distributions is restricted by the way to model the probability density,
many works have been established to alleviate this problem and achieved appealing results \cite{makhzani2016adversarial,tolstikhin2018wasserstein,razavi2019generating, dinh2016density, papamakarios2017masked, kingma2018glow, behrmann19a}. As a representative of implicit generative models, generative adversarial networks (GANs) use a min-max game objective to learn the target distribution. It has been shown that vanilla GAN \cite{goodfellow14} minimizes the Jensen-Shannon divergence between the target distribution and the generated distribution.
To generalize vanilla GAN, researchers consider some other criterions including more general $f$-divergences \cite{nowozin16}, 1-Wasserstein distance \cite{arjovsky17} and maximum mean discrepancy (MMD) \cite{binkowski18}. Meanwhile, recent progress on designing network architectures \cite{radford15, zhang18} and training techniques \cite{karras18, brock2018large} has enabled GANs to produce impressive high-quality images.


Despite the extraordinary performance of generative models \citep{razavi2019generating, kingma2018glow, brock2018large, karras2019style}, there still exists a gap between the empirical success and the theoretical justification of these methods. For likelihood-based models, consistency results require that the data distribution is within the model family, which is often hard to hold in practice \cite{kingma14}.
Recently, new generative models have been developed from different perspectives, such as gradient flow in a measure space in which GAN can be covered as a special case \citep{gao2019deep, arbel2019maximum} and stochastic differential equations (SDE) \citep{song2019generative, song2020improved, song2021scorebased}. To push a simple initial distribution to the target one, however, these methods \citep{gao2019deep, arbel2019maximum, liutkus2019sliced, song2019generative, song2020improved,block2020generative} require the evolving time to go to infinity at the population level.
Therefore,
these methods require a strong assumption to achieve model consistency: the target must be log-concave or satisfy the log-Sobolev inequality.

To fill the gap, we propose a Schr\"{o}dinger Bridge approach to learn generative models.
Schr\"{o}dinger Bridge tackles the problem by interpolating a reference distribution to a target distribution based on the Kullback-Leibler divergence. The Schr\"{o}dinger Bridge can be formulated  via an SDE on a finite time interval $[0,1]$ with a time-varying drift term. At the population level, we can solve the SDE  using the  standard Euler-Maruyama method.
At the sample level, we derive our Schr\"{o}dinger Bridge algorithm by plugging the drift term into the Euler-Maruyamma method, where the drift term can be accurately estimated by a deep score network. 
The major contributions of this work are as follows:
\begin{itemize}
\item From the theoretical perspective, we prove the consistency of the Schr\"{o}dinger Bridge approach under the some mild smoothness assumptions of the target distribution. Our theory guarantees that the learned distribution converges to the target. 
To achieve model consistency, existing theories rely on strong assumptions, e.g., the target must be log-concave or satisfy some error bound conditions, such as the log-Sobolev inequality. These assumptions may not hold in practice.
\item From the algorithmic perspective, we develop a novel two-stage approach to make the theory of Schr\"{o}dinger Bridge work in practice, where the first stage effectively learns a smoothed version of the target distribution and the second stage drives the smoothed one to the target distribution. Figure \ref{demo} gives an overview of our two-stage algorithm.
\item Through synthetic data, we demonstrate that our Schr\"{o}dinger Bridge approach can stably learn multimodal distribution, while GANs are often highly unstable and prone to miss modes \cite{che17}. We also show that the proposed approach achieves comparable performance with state-of-the-art GANs on benchmark data.
\end{itemize}
In summary, we believe that our work suggests a new formulation of generative models.
\begin{figure}[ht]
\vskip 0.2in
\begin{center}
\centerline{\includegraphics[width=0.9\columnwidth]{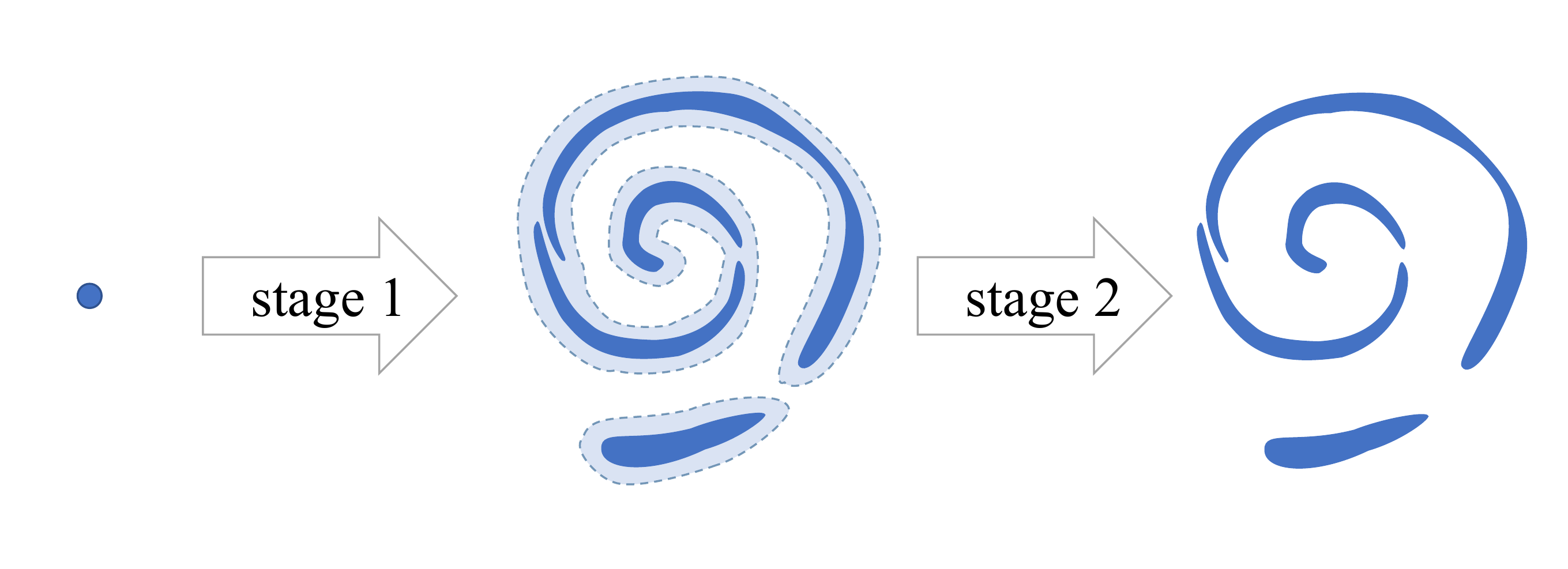}}
\caption{Overview of our two-stage algorithm. Stage 1 drives samples at $\mathbf{0}$ (left) to a smoothed data distribution (middle), and stage 2 learns the underlying target data distribution (right) with samples produced by stage 1. Stage 1 and stage 2 are achieved through the two different Schrodinger Bridges with theoretically guaranteed performance.}
\label{demo}
\end{center}
\vskip -0.2in
\end{figure}
\section{Background}
Let's first recall some background on  Schr\"{o}dinger Bridge problem adopted from \cite{leonard2014survey,chen2020stochastic}.

Let $\Omega = C([0,1],\mathbb{R}^d)$ be the space of $\mathbb{R}^d$-valued continuous functions on time interval $[0, 1]$. Denote $X = (X_t)_{t\in [0,1]}$ as the canonical process on $\Omega$, where $X_t(\omega) = \omega_t$, $\omega = (\omega_s)_{s\in [0,1]}\in \Omega$.
The canonical $\sigma$-field on $\Omega$ is then generated as $\mathscr{F}  = \sigma(X_t,t\in[0,1]) = \left\{\{\omega:(X_t(\omega))_{t\in [0,1]}\in H\}:H\in\mathscr{B}(\mathbb{R}^d)\right\}$. Denote $\mathcal{P}(\Omega)$ as the space of probability measures on the path space $\Omega$, and $\mathbf{W}^{\mx}_\tau\in\mathcal{P}(\Omega)$ as the Wiener measure with variance $\tau$ whose initial marginal is $\delta_{\mx}$. The law of the reversible Brownian motion, is then defined as $\mathbf{P}_{\tau} = \int \mathbf{W}^{\mx}_\tau\mathrm{d}\mx$, which is an unbounded measure on $\Omega$. One can observe that, $\mathbf{P}_{\tau}$ has a marginal coincides with the Lebesgue measure $\mathscr{L}$ at each $t$.

\citet{schrodinger1932theorie} studied the problem of finding the most likely random evolution between two
continuous probability distributions $\mu, \nu \in \mathcal{P}(\mathbb{R}^d)$. Nowadays,  people call the study of  Schr\"{o}dinger as the Schr\"{o}dinger  Bridge problem (SBP). In fact,
SBP can be further formulated as seeking a probability law on a path space that interpolates between $\mu$ and $\nu$, such that the probability law is close to the prior law of the Brownian diffusion in the sense of relative entropy \cite{jamison1975markov,leonard2014survey}, i.e.,
 finding  a path measure $\mathbf{Q}^* \in \mathcal{P}(\Omega)$ with marginal $\mathbf{Q}^*_{t} = (X_t)_{\#}\mathbf{Q}^*=\mathbf{Q}^*\circ X_t^{-1}, t\in [0,1]$ such that
$$\mathbf{Q}^* \in {\arg\min}_{\mathbf{Q}\in \mathcal{P}(\Omega)} \mathbb{D}_{\mathrm{KL}}(\mathbf{Q}||\mathbf{P}_{\tau}),$$ and  $$\mathbf{Q}_0 = \mu, \mathbf{Q}_1 = \nu,$$ where $\mu, \nu \in \mathcal{P}(\mathbb{R}^d)$,  relative entropy $\mathbb{D}_{\mathrm{KL}}(\mathbf{Q}||\mathbf{P}_{\tau}) = \int \log(\frac{d \mathbf{Q}}{d \mathbf{P}_{\tau}}) d \mathbf{Q} $ if $\mathbf{Q}\ll \mathbf{P}_{\tau}$ (i.e. $\mathbf{Q}$ is absolutely continuous w.r.t. $\mathbf{P}_{\tau}$), and $\mathbb{D}_{\mathrm{KL}}(\mathbf{Q}||\mathbf{P}_{\tau}) = \infty$ otherwise.
The following results characterize the solution to SBP.
\begin{thm}\label{th01}\cite{leonard2014survey}
If $\mu, \nu \ll \mathscr{L}$,  then SBP admits a unique solution $\mathbf{Q}^* = f^*(X_0)g^*(X_1)\mathbf{P}_{\tau}$, where
$f^*$, $g^*$ are $\mathscr{L}$-measurable nonnegative  functions on $\mathbb{R}^d$ satisfying the  Schr\"{o}dinger system
$\left\{\begin{array}{l}
 f^*(\mx) \mathbb{E}_{\mathbf{P}_{\tau}}\left[g^*\left(X_{1}\right) \mid X_{0}=\mx\right]= \frac{\mathrm{d} \mu}{\mathrm{d}\mathscr{L}}(\mx), \quad \mathscr{L}-a . e . \\
 g^*(\my)  \mathbb{E}_{\mathbf{P}_{\tau}}\left[f^{*}\left(X_{0}\right) \mid X_{1}=\my\right]=\frac{\mathrm{d} \nu}{\mathrm{d}\mathscr{L}}(\my), \quad \mathscr{L}-a . e .
 \end{array}\right.$
\end{thm}

Besides $\mathbf{Q}^*$, we can also characterize the density of the time-marginals of $\mathbf{Q}^*$, i.e. $\frac{d\mathbf{Q}_t^{*}}{d\mathscr{L}}(\mx)$.

Let $q(\mx)$ and $p(\my)$ be the density of $\mu$ and $\nu$ respectively, and $h_{\tau}(s, \mx, t, \my) = [2\pi\tau(t-s)]^{-d/2}\exp\left(-\frac{\|\mx - \my\|^2}{2\tau(t-s)}\right)$ be the transition density of the Wiener process. Then we have
$\mathbb{E}_{\mathbf{P}_{\tau}}\left[f^{*}\left(X_{0}\right) \mid X_{1}=\my\right] = \int h_{\tau}(0, \mx, 1, \my)f_{0}(\mx) \mathrm{d} \mx, \\
\mathbb{E}_{\mathbf{P}_{\tau}}\left[g^*\left(X_{1}\right) \mid X_{0}=\mx\right] = \int h_{\tau}(0, \mx, 1, \my)g_{1}(\my) \mathrm{d} \my$.
The above Schr\"{o}dinger system is equivalent to
$$\left\{\begin{array}{l}
 f^*(\mx) \int h_{\tau}(0, \mx, 1, \my)g_{1}(\my) \mathrm{d} \my= q(\mx), \\
 g^*(\my)  \int h_{\tau}(0, \mx, 1, \my)f_{0}(\mx) \mathrm{d} \mx= p(\my).
 \end{array}\right.$$
Denote $f_{0}(\mx) = f^*(\mx), \ \ g_{1}(\my) = g^*(\my),$
\begin{align*}
{f_{1}}(\my) = \int h_{\tau}(0, \mx, 1, \my)f_{0}(\mx) \mathrm{d} \mx,\\
{g_{0}}(\mx) = \int h_{\tau}(0, \mx, 1, \my)g_{1}(\my) \mathrm{d} \my.
\end{align*}
The  Schr\"{o}dinger system  in Theorem \ref{th01} can also be characterized
  by
  \begin{equation}\label{sbs}
  q(\mx) = f_0(\mx) {g_{0}}(\mx), \ \  p(\my)=  {f_{1}}(\my)g_1(\my)
  \end{equation}
  with the following  forward and backward time harmonic equations  \cite{chen2020stochastic}
 $$\left\{\begin{array}{l}
\partial_t f_t(\mx) = \frac{\tau\Delta}{2} f_t(\mx),  \\
\partial_t g_t(\mx) = -\frac{\tau\Delta}{2} g_t(\mx),
\end{array}\right. \quad \text { on }(0,1) \times \mathbb{R}^{d}.
$$
 Let $q_t$ denote marginal density of $\mathbf{Q}_t^{*}$, then it can be represented \cite{chen2020stochastic} by the product of  $g_t$  and $f_t$ defined as  $q_t(\mx) = \frac{d\mathbf{Q}_t^{*}}{d\mathscr{L}}(\mx)$, and $q_t(\mx) = f_t(\mx)g_t(\mx)$.

There are also dynamic formulations of SBP. Let  $\mathcal{U}$ consist of admissible Markov controls with finite energy.
 The following theorem shows that, the vector field
 \begin{equation}\label{drift}
 \begin{aligned}
 &\mathbf{u}_{t}^* = \tau \mv^*_{t} = \tau \nabla_{\mx}\log g_t(\mx) \\
 = &\tau \nabla_{\mx}\log  \int h_{\tau}(t, \mx, 1, \my)g_1(\my) \mathrm{d} \my
 \end{aligned}
 \end{equation}
  solves such a stochastic control problem:
\begin{thm}\label{th02}\cite{dai1991stochastic}
$$\mathbf{u}^*_{t}(\mx)\in \arg\min_{\mathbf{u} \in \mathcal{U}}\mathbb{E}\left[\int_0^1\frac{1}{2}\|\mathbf{u}_t\|^2\mathrm{d}t\right]$$
s.t.
\begin{equation}\label{sde}
\left\{\begin{array}{l}
\mathrm{d}\mx_t = \mathbf{u}_t \mathrm{d}t + \sqrt{\tau}\mathrm{d}\mw_t, \\
\mx_0\sim q(\mx),\quad \mx_1\sim p(\mx).
\end{array}\right.
\end{equation}
\end{thm}
According to Theorem \ref{th02}, the dynamics determined by the SDE in (\ref{sde}) with a time-varying drift term $\mathbf{u}^*_{t}$ in (\ref{drift}) will make the particles sampled from the initial distribution $\mu$ evolve to the particles drawn from the target distribution $\nu$ in the unit time interval. This nice property is what we need in generative learning because we want to learn the underlying target distribution $\nu$ via pushing forward a simple reference distribution $\mu$. Theorem \ref{th02} also indicates that such a solution has minimum energy in terms of quadratic cost.

\section{Generative Learning via Schr\"{o}dinger Bridge}
In generative learning, we observe i.i.d. data $\mx_1, ..., \mx_n$ from an unknown distribution $p_{\mathrm{data}}\in \mathcal{P}(\mathbb{R}^d)$.
The underlying distribution  $p_{\mathrm{data}}$ often has multi-modes or lies on  a low-dimensional manifold, which may cause difficulty to learn from simple distribution such  as Gaussian or Dirac measure supported on a single point.  To make the generative learning    task easy to handle,   we can first learn a smoothed version of $p_{\mathrm{data}}$ from the simple reference distribution, say $$q_{\sigma}(\mx) = \int p_{\mathrm{data}}(\my)\Phi_{\sigma}(\mx-\my)\mathrm{d}\my,$$ where  $\Phi_{\sigma}(\cdot)$ is the density of $\mathscr{N}(\mathbf{0}, \sigma^2\mI)$, the  variance of Gaussian noise  $\sigma^2$  controls the smoothness of $q_{\sigma}$. Then we learn $p_{\mathrm{data}}$ starting from $q_{\sigma}$.
At the population level, this idea can be done via Schr\"{o}dinger Bridge from the point of view of the stochastic control problem (Theorem \ref{th02}).
%
 To be precise, we have the following theorem.
\begin{thm}
\label{thm1}
Define the density ratio $f(\mx)=\frac{q_{\sigma}(\mx)}{\Phi_{\sqrt{\tau}}(\mx)}$.
Then for the SDE
\begin{equation}
\label{stg1}
\mathrm{d}\mx_t = \tau\nabla\log\mathbb{E}_{\mz\sim\Phi_{\sqrt{\tau}}}[f(\mx_t+\sqrt{1-t}\mz)] \mathrm{d}t+\sqrt{\tau}\mathrm{d}\mw_t
\end{equation}
with initial condition $\mx_0 = \mzero$, we have $\mx_1 \sim q_{\sigma}(\mx)$.

And, for the SDE
\begin{equation}
\label{stg2}
\mathrm{d}\mx_t = \sigma^2\nabla\log q_{\sqrt{1-t}\sigma}(\mx_t)\mathrm{d}t+\sigma\mathrm{d}\mw_t
\end{equation}
with initial condition $\mx_0 \sim q_{\sigma}(\mx)$, we have $\mx_1 \sim p_{\mathrm{data}}(\mx)$.
\end{thm}

According to Theorem \ref{thm1}, at the population level, the target $p_{\mathrm{data}}$ can be learned from the Dirac mass supported at $\mathbf{0}$  through two SDEs (\ref{stg1}) and  (\ref{stg2}) in the unit time interval [0,1].
The  main feature of the SDEs (\ref{stg1}) and  (\ref{stg2}) is that both  drift terms are time-varying,  which is different
classical Langevin SDEs  with time-invariant drift terms \citep{song2019generative,song2020improved}.   The benefit of time-varying drift terms is that
the dynamics  in   (\ref{stg1}) and  (\ref{stg2}) will push the initial distributions to the target distributions in
a unit time interval, while the classical Langevin SDE needs time to go to infinity.

\subsection{Estimation of the drift terms}
Based on Theorem \ref{thm1}, we can  run the Euler-Maruyama method to solve the  SDEs (\ref{stg1}) and  (\ref{stg2}) and get particles  that
approximately drawn from the targets \cite{higham2001algorithmic}. However,  the drift terms in Theorem \ref{thm1} depend on the underlying target.
To make the Euler-Maruyama method practical, we need to estimate the two drift terms in  (\ref{stg1}) and  (\ref{stg2}).
 In Eq. (\ref{stg1}), some calculation shows that
\begin{align}\label{estdrf1}
& \nabla\log\mathbb{E}_{\mz\sim\Phi_{\sqrt{\tau}}}[f(\mx+\sqrt{1-t}\mz)] \nonumber \\
= & \frac{\mathbb{E}_{\mz\sim\Phi_{\sqrt{\tau}}}\left[f(\mx+\sqrt{1-t}\mz)\nabla\log f(\mx+\sqrt{1-t}\mz)\right]}{\mathbb{E}_{\mz\sim\Phi_{\sqrt{\tau}}}[f(\mx+\sqrt{1-t}\mz)]},
\end{align}
and
$$\nabla\log f(\mx) = \nabla\log q_{\sigma}(\mx) + \mx / \tau.$$
Let $\hat{f}$ and $\widehat{\nabla\log q_{\sigma}}$ be the estimators of the density ratio $f$  and the score of
 $q_{\sigma}(\mx)$, respectively. After plugging them into (\ref{estdrf1}),
 we can obtain an estimator of the drift term
  in  (\ref{stg1}) by computing the expectation with Monte Carlo approximation.


Now we consider obtaining the estimator of density ratio $\hat{f}$, via minimizing the logistic regression loss
$
\mathcal{L}_{\mathrm{logistic}}(r)
= \mathbb{E}_{q_{\sigma}(\mx)}\log (1+\exp(-r(\mx)))
 + \mathbb{E}_{\Phi_{\sqrt{\tau}}(\mx)}\log (1+\exp(r(\mx)))
$. By setting the first variation to zero, the optimal solution is given by
\begin{align*}
r^*(\mx)=\log\frac{q_{\sigma}(\mx)}{\Phi_{\sqrt{\tau}}(\mx)}.
\end{align*}
Therefore, given samples $\widetilde{\mx}_1,...,\widetilde{\mx}_n$ from $q_{\sigma}(\mx)$, which can be obtained by adding Gaussian noise drawn from $\Phi_{\sigma}$  on $\mx_1,..., \mx_n \sim p_{\mathrm{data}}$,  and samples $\mz_1,...,\mz_n$ from  $\Phi_{\sqrt{\tau}}$, we can estimate the density ratio $f(\mx)$ by
\begin{equation}\label{edrift1}
\hat{f}(\mx) = \exp(\hat{r}_{\phi}(\mx)),
\end{equation}
where $\hat{r}_{\phi} \in \mathcal{NN}_{\phi}$ is the neural network that minimizes the empirical loss:
\begin{align}\label{drest}
\hat{r}_{\phi}  \in {\arg\min}_{r_{\phi}\in \mathcal{NN}_{\phi}}\frac{1}{n}\sum_{i =1}^n[&\log(1+\exp(-r_{\phi}(\widetilde{\mx}_i))) \nonumber\\
& + \log(1+\exp(r_{\phi}(\mz_i)))].
\end{align}

Next, we consider estimating the time-varying drift term  in  (\ref{stg2}),  i.e., $\nabla\log q_{\sqrt{1-t}\sigma}(\mx)$ for $t \in [0, 1]$. To do so, we build a deep network as the score estimator for $\nabla\log q_{\tilde{\sigma}}(\mx)$ with $\tilde{\sigma}$ varying in  $[0,\sigma]$. 
\citet{vincent2011connection} showed that, explicitly matching the score by minimizing the objective
$$\frac12\mathbb{E}_{q_{\tilde{\sigma}}(\mx)}\|\ms_{\theta}(\mx, \tilde{\sigma})-\nabla_{\mx}\log q_{\tilde{\sigma}}(\mx)\|^2$$
is equivalent to minimizing the denoising score matching objective
\begin{align*}
&\frac12\mathbb{E}_{p_{\mathrm{data}}(\mx)}\mathbb{E}_{\mathscr{N}(\tilde{\mx}; \mx, \tilde{\sigma}^2\mI)}\|\ms_{\theta}(\tilde{\mx}, \tilde{\sigma})-\nabla_{\tilde{\mx}}\log q_{\tilde{\sigma}}(\tilde{\mx}|\mx)\|^2\\
=&\frac12\mathbb{E}_{p_{\mathrm{data}}(\mx)}\mathbb{E}_{\mathscr{N}(\tilde{\mx}; \mx, \tilde{\sigma}^2\mI)}\left\|\ms_{\theta}(\tilde{\mx}, \tilde{\sigma})+\frac{\tilde{\mx}-\mx}{\tilde{\sigma}^2}\right\|^2.
\end{align*}
Thus we build the score estimator following \citet{song2019generative,song2020improved} as
\begin{equation}\label{dse1}
\hat{\ms}_{\theta}(\cdot,\cdot)  \in \arg\min_{\ms_{\theta} \in \mathcal{NN}_{\theta}} {\mathcal{L}}(\theta),
\end{equation}
 \begin{equation}\label{dse2}
 {\mathcal{L}}(\theta)=  \frac{1}{m}\sum_{j=1}^{m}\lambda(\tilde{\sigma}_{j}){\mathcal{L}}_{\tilde{\sigma}_j}(\theta),
\end{equation}
$${\mathcal{L}}_{\tilde{\sigma}_j}(\theta) = \sum_{i=1}^n \left\|\ms_{\theta}(\mx_i+\mz_i,\tilde{\sigma})+\frac{\mz_i}{\tilde{\sigma}^2_j}\right\|^{2}/n,$$
variance terms $\tilde{\sigma}_j^2,j=1,\dots,m$ are i.i.d. samples from $\mathrm{Uniform}[0,\sigma^2]$ with sample size $m$,  $\lambda(\tilde{\sigma})=\tilde{\sigma}^2$ is a nonnegative scaling factor to ensure all the  summands in (\ref{dse2})
have the same scale, and $\mz_i, i = 1,...,n$ are i.i.d. from $\Phi_{\tilde{\sigma}}$.

At last, we establish the consistencies of the deep density ratio estimator $\hat{f}(\mx) = \exp(\hat{r}_{\phi}(\mx))$  and the deep score estimator $\widehat{\nabla\log q_{\tilde{\sigma}}}(\mx) = \hat{\ms}_{\theta}(\mx;\tilde{\sigma})$ in Theorem \ref{thdes1} and Theorem \ref{thdes2}, respectively.
\begin{thm}\label{thdes1}
Assume that the  support of  $p_{\mathrm{data}}(\mx)$  is   contained in a compact set, and $f(\mx)$ is Lipschitz continuous and bounded.  Set the depth $\mathcal{D}$, width $\mathcal{W}$, and size $\mathcal{S}$ of  $\mathcal{NN}_{\phi} $
as $$\mathcal{D}=\mathcal{O} (\log (n)), \mathcal{W}= \mathcal{O}( n^{\frac{d}{2(2+d)}} / \log (n)),$$  $$\mathcal{S}=\mathcal{O}(n^{\frac{d-2 }{d+2 }} \log (n)^{-3}).$$ Then $\mathbb{E}[\|\hat{f}(\mx)-f(\mx)\|_{L^2(p_{\mathrm{data}})}]\rightarrow0$  as $n\rightarrow \infty.$
\end{thm}
\begin{thm}\label{thdes2}
Assume that  $p_{\mathrm{data}}(\mx)$ is differentiable with  bounded support, and $ \nabla\log q_{\tilde{\sigma}}(\mx)$ is Lipschitz continuous and bounded for $(\tilde{\sigma},\mx)\in [0,\sigma] \times \mathbb{R}^d$.  Set the depth $\mathcal{D}$, width $\mathcal{W}$, and size $\mathcal{S}$ of  $\mathcal{NN}_{\theta} $
as $$\mathcal{D}=\mathcal{O} (\log (n)), \mathcal{W}= \mathcal{O}( \max\{n^{\frac{d}{2(2+d)}} / \log (n), d\}),$$  $$\mathcal{S}=\mathcal{O}(d n^{\frac{d-2 }{d+2 }} \log (n)^{-3}).$$ Then
$\mathbb{E}[\| \|\widehat{\nabla\log q_{\tilde{\sigma}}}(\mx)-\nabla\log q_{\tilde{\sigma}}(\mx)\|_2\|_{L^2(q_{\tilde{\sigma}})}]\rightarrow0$  as $m,n\rightarrow \infty.$
\end{thm}
\subsection{Schr\"{o}dinger Bridge Algorithm}
With the two estimators $\hat{f}$ and $\widehat{\nabla\log q_{\tilde{\sigma}}}$
, we can use the Euler-Maruyama method to approximate numerical solutions of SDEs (\ref{stg1}) and (\ref{stg2}).
Let $N_1$ and $N_2$  be the number of uniform grids on the time interval  $[0, 1]$. In stage 1, we start from $\mathbf{0}$ and run   Euler-Maruyama for (\ref{stg1}) with the estimated $\hat{f}$ and $\widehat{\nabla\log q_{\sigma}}$ in the drift term to obtain samples that follow $q_{\sigma}$ approximately. In stage 2, we start with the samples from $q_{\sigma}$ and run another Euler-Maruyama  for (\ref{stg2}) with the estimated time-varying drift term $\widehat{\nabla\log q_{\tilde{\sigma}}}$.
We summarize our two-stage Schr\"{o}dinger Bridge   algorithm in \ref{alg:sampling1}.
\begin{algorithm}[tb]
   \caption{Sampling}
   \label{alg:sampling1}
   \begin{algorithmic}
   \STATE {\bfseries Input:} $\hat{f}(\cdot)$, $\hat{\ms}_{\theta}(\cdot, \cdot)$, $\tau$, $\sigma$, $N_1$, $N_2$, $N_3$
   \end{algorithmic}
   \begin{tcolorbox}[colframe=peach-orange, colback=peach-orange, boxrule=0pt,arc=5pt,left=0pt,right=0pt,top=0pt,bottom=0pt,boxsep=0pt]
   {\parbox{8cm} {\vbox{
   \begin{algorithmic}
   \STATE Initialize particles as $\mx_0 = \mzero$
   \RIGHTCOMMENT{\textcolor{aurometalsaurus}{stage 1}}
   \FOR{$k=0$ {\bfseries to} $N_1-1$}
   \STATE Sample $\{\mz_i\}_{i=1}^{2N_3}, \meps_k \sim \mathscr{N}(\mathbf{0}, \mI)$
   \STATE $\tilde{\mx}_i = \mx_k+\sqrt{\tau\left(1-\frac{k}{N_1}\right)}\mz_i$, $i = 1,..., N_3$
   \STATE $\mb(\mx_k)=\frac{\sum_{i=1}^{N_3} \hat{f}(\tilde{\mx}_i)[\hat{\ms}_{\theta}(\tilde{\mx}_i, \sigma)+\sqrt{\left(1-\frac{k}{N_1}\right)/\tau}\mz_i]}{\sum_{i=N_3+1}^{2N_3} \hat{f}(\tilde{\mx}_i)} + \frac{\mx_k}{\tau}$.
   \STATE $\mx_{k+1} = \mx_k + \frac{\tau}{N_1}\mb(\mx_k)+\sqrt{\frac{\tau}{N_1}}\meps_k.$
   \ENDFOR
   \end{algorithmic}}}}\end{tcolorbox}

   \begin{tcolorbox}[colframe=piggypink, colback=piggypink, boxrule=0pt,arc=5pt,left=0pt,right=0pt,top=0pt,bottom=0pt,boxsep=0pt]
   {\parbox{8cm} {\vbox{
   \begin{algorithmic}
   \STATE Set $\mx_0 = \mx_{N_1}$
   \RIGHTCOMMENT{\textcolor{aurometalsaurus}{stage 2}}
   \FOR{$k=0$ {\bfseries to} $N_2-1$}
   \STATE Sample $\meps_k \sim \mathscr{N}(\mathbf{0}, \mI)$
   \STATE $\mb(\mx_k)=\hat{\ms}_\theta(\mx_k, \sqrt{1-\frac{k}{N_2}}\sigma)$
   \STATE $\mx_{k+1} = \mx_k + \frac{\sigma^2}{N_2}\mb(\mx_n)+\frac{\sigma}{\sqrt{N_2}}\meps_k$
   \ENDFOR
   \end{algorithmic}}}}
   \end{tcolorbox}
   \begin{algorithmic}
   \STATE {\bfseries return} $\mx_{N_2}$
   \end{algorithmic}
\end{algorithm}

Interestingly, the second stage of our proposed Schr\"{o}dinger Bridge algorithm  \ref{alg:sampling1} recovers the reverse-time Variance Exploding (VE) SDE algorithm proposed in  \citet{song2021scorebased}, if their annealing scheme is chosen to be linear as $\sigma^2(t) = \sigma^2 \cdot t$.
 From this point of view, our Schr\"{o}dinger Bridge algorithm
also provides deeper understanding of annealing score based sampling, i.e., the reverse-time VE SDE algorithm (with a proper annealing scheme) proposed by
\citet{song2021scorebased} is equivalent to the Schr\"{o}dinger Bridge SDE (\ref{stg2}).

\subsection{Consistency of Schr\"{o}dinger Bridge Algorithm}
Let $$ D_1(t,\mx) =  \nabla\log\mathbb{E}_{\mz\sim\Phi_{\sqrt{\tau}}}[f(\mx+\sqrt{1-t}\mz)],$$
$$ D_2(t,\mx) =  \nabla\log q_{\sqrt{1-t}\sigma}(\mx)$$
be the drift terms. Denote $$h_{\sigma,\tau}(\mx_1,\mx_2) = \exp{\left(\frac{\|\mx_1\|^2}{2\tau}\right)}p_{\mathrm{data}}(\mx_1 + \sigma \mx_2).$$
Now we establish the consistency of our Schr\"{o}dinger Bridge Algorithm which can drive a simple distribution to the target one. To this end, we need the following assumptions:
\begin{assumption}\label{A1}
$\mathrm{supp}(p_{\mathrm{data}})$ is contained in a ball with radius $R$, and $p_{\mathrm{data}}>c>0$ on its support.
\end{assumption}
\begin{assumption}\label{A2}
$\|D_i(t,\mx)\|^2 \leq C_1(1+\|\mx\|^2)$, $\forall \mx \in \mathrm{supp}(p_{\mathrm{data}})$, $t\in[0,1]$, where $C_1\in\mathbb{R}$ is a constant.
\end{assumption}
\begin{assumption}\label{A3}
$\|D_i(t_1,\mx_1) - D_i(t_2,\mx_2)\| \leq C_2(\|\mx_1-\mx_2\|+|t_1-t_2|^{1/2})$, $\forall \mx_1,\mx_2 \in \mathrm{supp}(p_{\mathrm{data}}), t_1, t_2 \in[0,1]$.  $C_2\in\mathbb{R}$ is another constant.
\end{assumption}
\begin{assumption}\label{A4}
$h_{\sigma,\tau}(\mx_1,\mx_2)$, $\nabla_{\mx_1}h_{\sigma,\tau}(\mx_1,\mx_2)$, $p_{\mathrm{data}}$ and $\nabla p_{\mathrm{data}}$ are $L$-Lipschitz functions.
\end{assumption}

\begin{thm}\label{thm2}
Under Assumptions 1-4,
$$\mathbb{E}[\mathcal{W}_2(\mathrm{Law}(\mx_{N_2}), p_{\mathrm{data}})] \rightarrow 0, \ \ \mathrm{as} \ \  n,N_1,N_2,N_3 \rightarrow \infty,$$
where $\mathcal{W}_2$ is the 2-Wasserstein distance between two distributions.
\end{thm}
The consistency of the proposed Schr\"{o}dinger Bridge is mainly based on mild assumptions (such as  smoothness  and boundedness) without some restricted technical requirements that the target distribution has to be log-concave or fulfill the log-Sobolev inequality \cite{gao2020generative, arbel2019maximum, liutkus2019sliced,block2020generative}.

\section{Related Work}\label{relate}
We discuss connections and differences between our Schr\"{o}dinger Bridge approach and existing related works.

Most of existing generative models, such as VAEs, GANs and flow-based methods, parameterize a transform map with a neural network  $G$ that minimizes 
an integral probability metric.
Clearly, they are quite different from our proposal.

 Recently, particle methods  derived in the perspective of gradient flows in measure spaces or SDEs have  been studied   \citep{johnson18,gao2019deep,arbel2019maximum,song2019generative,song2020improved,song2021scorebased}.
 Here we clarify the main differences of our  Schr\"{o}dinger Bridge approach and the above mentioned particle methods.
  The proposals in \citep{johnson18,gao2019deep,arbel2019maximum} are derived based on  the surrogate of the geodesic interpolation \citep{gao2020generative,liutkus2019sliced,song2019generative}. 
They utilize the invariant measure of SDEs to model the  generative task, resulting in an iteration scheme that looks similar to our  Schr\"{o}dinger Bridge. However, the main difference lies that the drift terms of the Langevin SDEs in  \citep{song2019generative,song2020improved,block2020generative} are time-invariant in contrast to the time-varying drift term in our formulation.
   As shown in Theorem \ref{thm1}, the benefit of the time-varying drift term is essential: the SDE of  Schr\"{o}dinger Bridge runs on a unit time interval $[0,1]$ will recover the target distribution at the terminal time. However, the evolution measures of the above mentioned methods \cite{gao2019deep,arbel2019maximum,song2019generative,song2020improved, block2020generative,gao2020generative} only converge to the target when the time goes to infinity. Hence, some technical requirements are imposed to the target distribution, such as log-concave or the log-Sobolev inequality,
    to guarantee the consistency of  Euler-Maruyama discretization. However, these assumptions may often be too strong to hold in real data analysis.
    We proposed a two-stage approach to make the Schr\"{o}dinger Bridge formulation work in practice. We drive the Dirac distribution to a smoothed version of underlying distribution $p_{\mathrm{data}}$ in stage 1 and then learn $p_{\mathrm{data}}$ from the smoothed version in stage 2. Interestingly, the second stage of the proposed Schr\"{o}dinger Bridge algorithm recovers the reverse-time Variance Exploding SDE algorithm (VE SDE) \cite{song2021scorebased} when their annealing scheme is linear, i.e., $\sigma^2(t)=\sigma^2\cdot t$. Therefore, the analysis developed here also provides a theoretical justification of why the reverse-time VE SDE algorithm works well. However, their setting is $\sigma^2(t)=(\sigma^2_{\max})^t\cdot(\sigma^2_{\min})^{1-t}$. This implies that the end-time distribution of the reverse-time VE SDE is still a smoothed one (with noise level $\sigma_{\min}$), resulting in a barrier of establishing the consistency. Another fundamental difference between our approach and reverse-time VE SDE is that, the reverse-time VE SDE also need a smoothed distribution as the input of theoretically, but they only approximately use large Gaussian noises as the initialization of the denoising process. Stage 1 ensures our algorithm to learn samples from the smoothed data distribution in unit time, which is necessary for model consistency.

\section{Experiments}
In this section, we first employ two-dimensional toy examples to show the ability of our algorithm to learn multimodal distributions which may not satisfy log-Sobolev inequality. Next, we show that our algorithm is able to generate realistic image samples. We also demonstrate the effective of our approach by image interpolation and image inpainting. We use two benchmark datasets including CIFAR-10 \cite{krizhevsky2009learning} and CelebA \cite{celebA}. For CelebA, the images are center-cropped and resized to $64\times64$. Both of the datasets are normalized by first rescaling the pixel values to $[0, 1]$, and then substracting a mean vector $\bar \mx$ estimated using 50,000 samples to center the data distributions at the origin. In our algorithm, the particles start from $\delta_{\mzero}$. To improve the performance, it is helpful to align the sample mean to the origin. After generation, we add the image mean $\bar\mx$ back to the generated samples. More details
on the hyperparameter settings and  network architectures, and some additional  experiments are provided in the supplementary material.

\subsection{Setup}
For the noise level $\sigma$, we set $\sigma = 1.0$ in this paper for generative tasks including both 2D example and CIFAR-10. In fact, the performance of our algorithm is insensitive to the choice of $\sigma$ when $\sigma$ is given in a reasonable range (the results with other $\sigma$ values are shown in the supplementary material). We find that the performance of our algorithm is often among the best by setting $\sigma = 1.0$ for $32\times32$ images. The reason is that a very small $\sigma$ can not make $q_{\sigma}$ smooth enough and harms the performance of stage 1 while a very large $\sigma$ brings more difficulty for our stage 2 to anneal the noise level down. For larger images like CelebA, as the dimensionality of samples is higher, we increase the noise level $\sigma$ to $2.0$.
We also compare the results by varying the value of the variance of the Wiener measure $\tau$ for image generation. The numbers of grids are chosen as $N_1=N_2=1,000$ for stage 1 and stage 2. We use sample size $N_3=1$ to estimate the drift term in stage 1 for both 2D toy examples and real images. In general, we find that a larger sample size $N_3$ does not significantly improve sample quality.

\subsection{Learning 2D Multimodal Distributions}
We demonstrate that our algorithm can effectively learn multimodal distributions. The distribution we adopt is a mixture of Gaussians with 6 components. Each of the components has a mean with a distance equaling to $5.0$ from the origin, and a variance $0.01$, as shown in Fig. \ref{2d-reals}. The components are relatively far away from each other. It is a very challenging task for GANs to learn this multimodal distribution because this distribution may not satisfy the log-Sobolev inequality. Fig. \ref{2d-gan} shows the failure of vanilla GAN, where several modes are missed. However, Fig. \ref{2d-stage1} and \ref{2d-stage2} show that our algorithm is able to stably generate samples from the multimodal distribution without ignoring any of the modes. In Fig. \ref{2d-v}, we compare the ground truth velocity fields induced by drift terms $D_1(t, \mx), D_2(t, \mx)$ with the estimated velocity fields at the end of each stage. Our estimated drift terms are close to the ground truth except for the region with nearly zero probability density.
\begin{figure}[ht]
\vskip 0.2in
\begin{center}
\subfigure[]{\includegraphics[width=0.24\columnwidth]{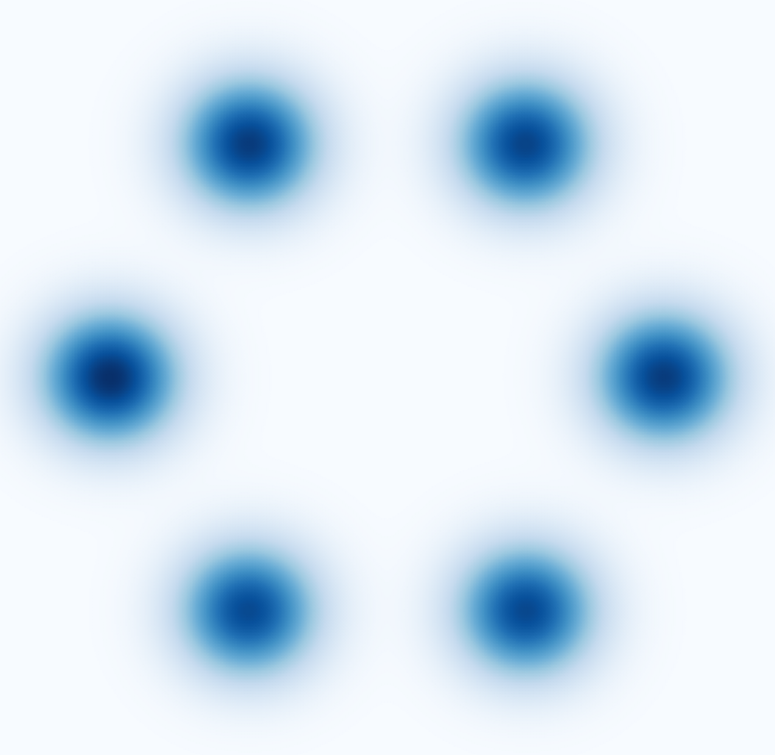}\label{2d-reals}}
\subfigure[]{\includegraphics[width=0.24\columnwidth]{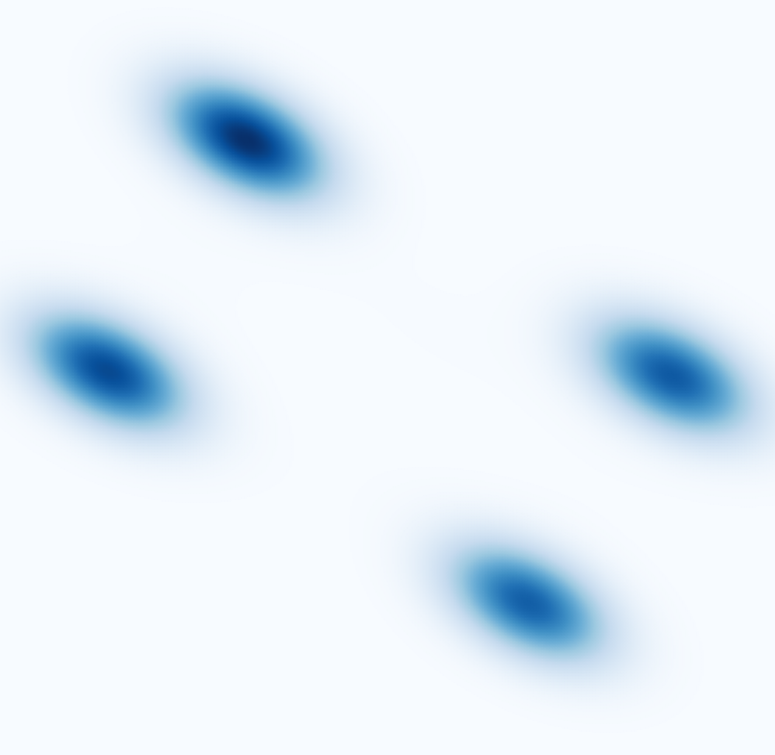}\label{2d-gan}}
\subfigure[]{\includegraphics[width=0.24\columnwidth]{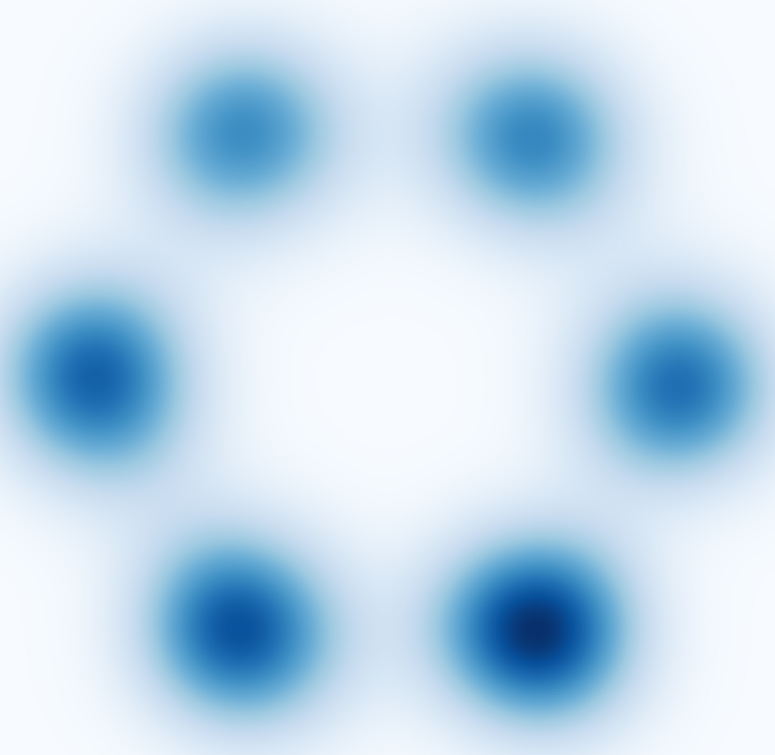}\label{2d-stage1}}
\subfigure[]{\includegraphics[width=0.24\columnwidth]{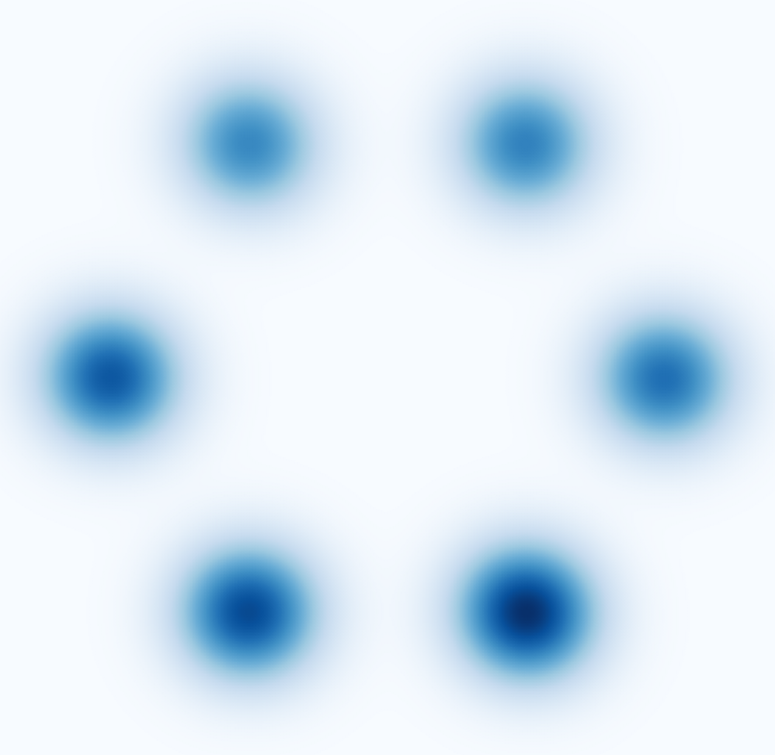}\label{2d-stage2}}
\caption{KDE plots for mixture of Gaussians with 5,000 samples. (a). Ground truth. (b). Distribution learned by vanilla GAN. (c). Distribution learned by the proposed method after stage 1 ($\tau=5.0$). (d). Distribution learned by the proposed method after stage 2.}
\label{2d}
\end{center}
\vskip -0.2in
\end{figure}

\begin{figure}[ht]
\vskip 0.2in
\begin{center}
\subfigure[]{\includegraphics[width=0.24\columnwidth]{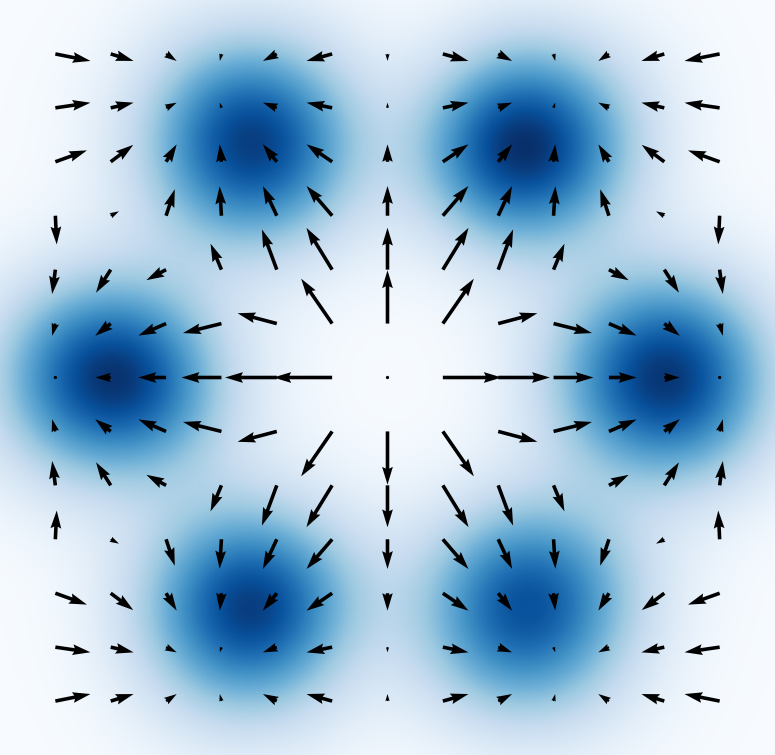}\label{real1}}
\subfigure[]{\includegraphics[width=0.24\columnwidth]{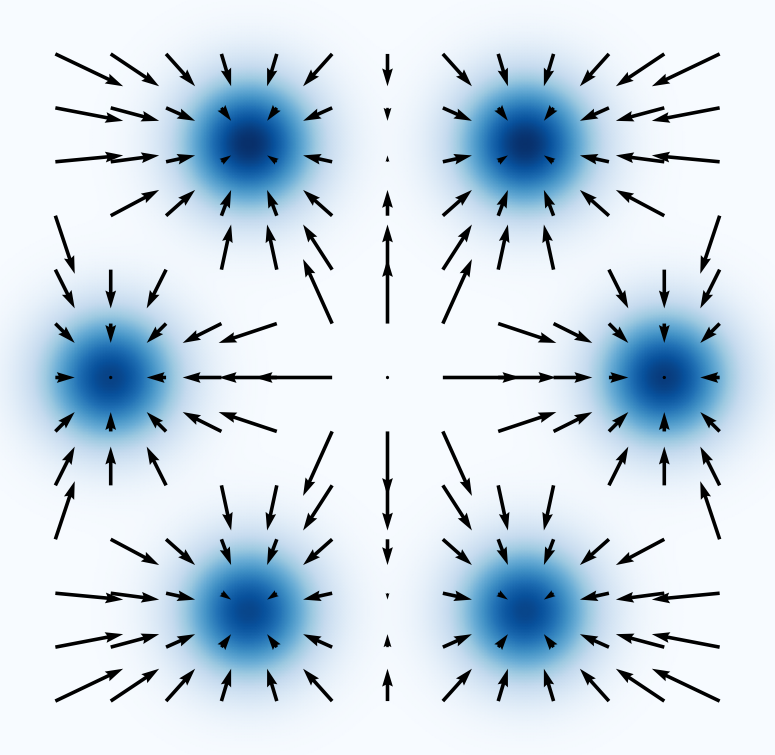}\label{real2}}
\subfigure[]{\includegraphics[width=0.24\columnwidth]{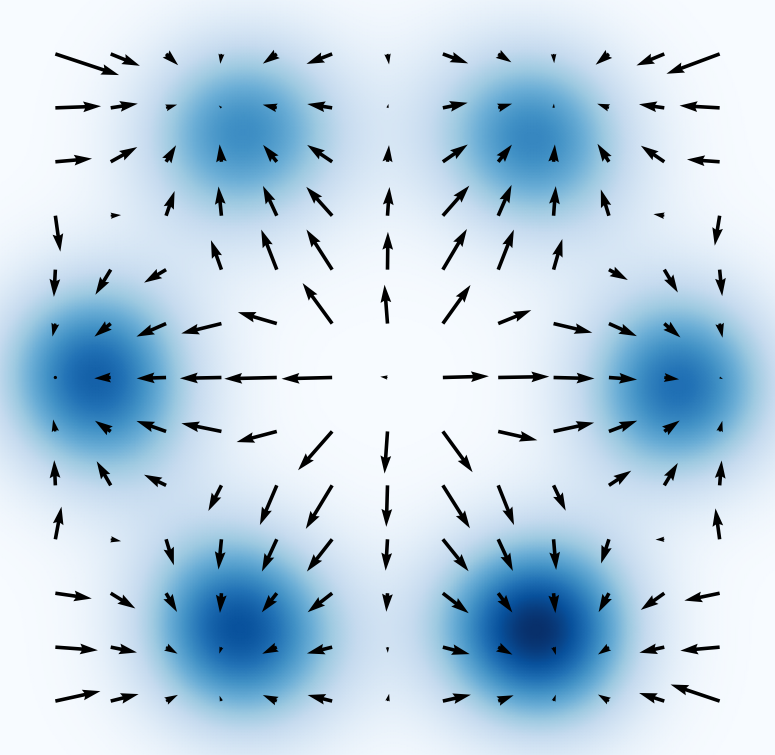}\label{est1}}
\subfigure[]{\includegraphics[width=0.24\columnwidth]{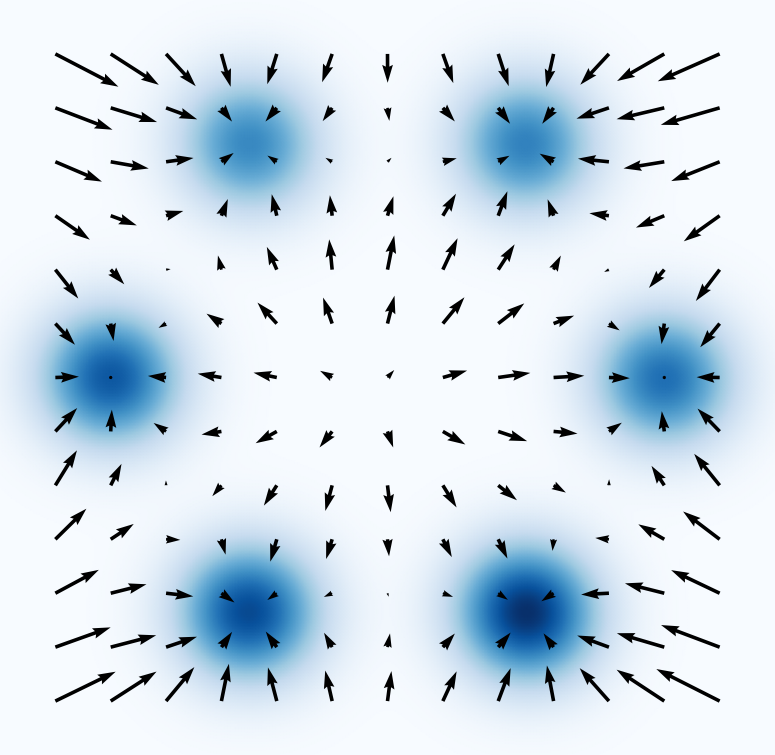}\label{est2}}
\caption{Velocity fields. (a) and (b). Ground truth velocity fields at the end of stages 1 and 2. (c) and (d). Estimated velocity fields at the end of stages 1 and 2.}
\label{2d-v}
\end{center}
\vskip -0.2in
\end{figure}

\subsection{Effectiveness of Two Stages for Image Generation}
Fig. \ref{particle-evolution} shows the particle evolution on CIFAR-10 in our algorithm, where the two stages are annotated with corresponding colors. It shows that our two-stage approach provides a valid path for the particles to move from the origin to the target distribution. A natural question is: what are the roles of stage 1 and stage 2 in the generative modeling, respectively?
In this subsection, we design experiments to answer this question.
\begin{figure}[ht]
\vskip 0.2in
\begin{center}
\centerline{\includegraphics[width=0.9\columnwidth]{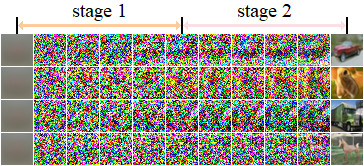}}
\caption{Particle evolution on CIFAR-10. The column in the center indicates particles obtained after stage 1.}
\label{particle-evolution}
\end{center}
\vskip -0.2in
\end{figure}

We first evaluate the role of stage 1. For this purpose, we skip stage 1 but simply run stage 2 using non-informative Gaussian noises as the initial condition.
Fig. \ref{effect-stage1} shows that the approach only using  stage 2 generates worse image samples than the proposed two-stage approach. These results indicate that the role of stage 1 is to provide a better initial reference for stage 2. 
The role of stage 2 is easier to check: it is a Schr\"{o}dinger Bridge from $q_{\sigma}(\mx)$ to the target distribution $p_{\mathrm{data}}(\mx)$.
In Fig. \ref{effect-stage2}, we perturb real images with Gaussian noises of variance $\sigma^2=1.0$. Our stage 2 anneals the noise level to zero and drives the particles to the data distribution. Moreover, Fig. \ref{effect-stage2} also indicates that stage 2 not only recover the original images, but also generate images with some extent of diversity.
\begin{figure}[ht]
\vskip 0.2in
\begin{center}
\subfigure[]{\includegraphics[width=0.24\columnwidth]{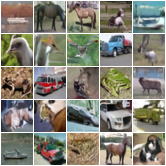}}
\subfigure[]{\includegraphics[width=0.24\columnwidth]{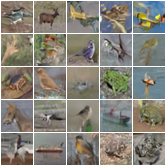}}
\subfigure[]{\includegraphics[width=0.24\columnwidth]{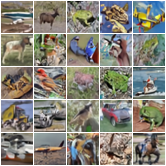}}
\subfigure[]{\includegraphics[width=0.24\columnwidth]{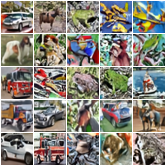}}
\caption{Comparison with random image samples. (a). Samples produced by our algorithm with $\tau = 2.0$ (FID = 12.32). (b), (c), (d). Samples produced by stage 2 taking Gaussian noises with variance $1.0$ (FID = 32.60), $1.5$ (FID = 24.76), $2.0$ (FID = 51.21) as input respectively.}
\label{effect-stage1}
\end{center}
\vskip -0.2in
\end{figure}
\begin{figure}[ht]
\vskip 0.2in
\begin{center}
\includegraphics[width=0.8\columnwidth]{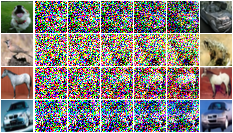}
\caption{Denoising with stage 2 for perturbed real images.}
\label{effect-stage2}
\end{center}
\vskip -0.2in
\end{figure}

\subsection{Results}
\label{results}

In this subsection, we evaluate our proposed approach on benchmark datasets. Fig. \ref{samples} presents the generated samples of our algorithm on CIFAR-10 and CelebA. Visually, our algorithm produces high-fidelity image samples which are competitive with real images. For quantitive evaluation, we employ Fréchet Inception Distance (FID) \cite{heusel17} and Inception Score (IS) \cite{salimans2016improved} to compare our method with other benchmark methods.

We first compare the FID and IS on CIFAR-10 dataset, with $\tau$ increasing from $1.0$ to $4.0$ using 50,000 generated samples.
Note that $\tau$ is the variance of the prior Wiener measure in stage 1, so it controls the behavior of the particle evolution from $\delta_{\mzero}$ to $q_{\sigma}$, and has an impact on the numerical results. To make the prior reasonable, we let $\tau_{\min}=\sigma^2=1.0$.
The reason is that, if the particles strictly follow the prior law of the Brownian diffusion with variance $\tau$ in stage 1, the end time marginal will be $\mathscr{N}(\mzero, \tau\mI)$. A good choice of the prior should make $\mathscr{N}(\mzero, \tau\mI)$ close to the end time marginal $q_{\sigma}$ which we are interested about.
As shown in Table \ref{fid-gamma}, our algorithm achieves the best performance at $\tau=2.0$. The results also indicate that our algorithm is stable with respect to the value of variance of the prior Wiener measure $\tau$ when $\tau \ge 2.0$. In general, reasonable choices of $\tau$ would result in relatively good generating performance.
\begin{table}[t]
\caption{FID and Inception Score on CIFAR-10 for $\tau \in [1, 4]$.}
\label{fid-gamma}
\vskip 0.15in
\begin{center}
\begin{small}
\begin{sc}
\begin{tabular}{lccccccc}
\toprule
$\tau$ & $1.0$ & $1.5$ & $2.0$ & $2.5$ \\
\midrule
FID & 37.20 & 20.49 &{\bf 12.32}&12.90\\
IS & 6.52 & 7.65 &{\bf 8.14}&7.99\\
\bottomrule
\toprule
$\tau$ & $3.0$ & $3.5$ & $4.0$ \\
\midrule
FID & 13.97&14.49&14.67\\
IS & 7.98&8.03&8.10\\
\bottomrule
\end{tabular}
\end{sc}
\end{small}
\end{center}
\vskip -0.1in
\end{table}

Table \ref{scores-cifar} presents the FID and IS of our algorithm evaluating with 50,000 samples, as well as other state-of-the-art generative models including WGAN-GP \cite{gulrajani17}, SN-SMMDGAN\cite{arbel18}, SNGAN \cite{miyato18}, NCSN \cite{song2019generative} and NCSNv2 \cite{song2020improved} on CIFAR-10. Our algorithm attains an FID score of 12.32 and an Inception Score of 8.14, which are competitive with the referred baseline methods. The quantitive results demonstrate the effectiveness of our algorithm.

\begin{figure}[ht]
\vskip 0.2in
\begin{center}
\subfigure[]{\includegraphics[width=0.45\columnwidth]{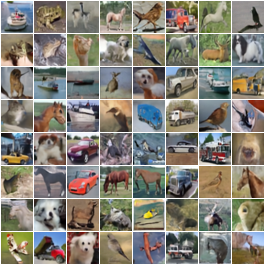}}
\subfigure[]{\includegraphics[width=0.45\columnwidth]{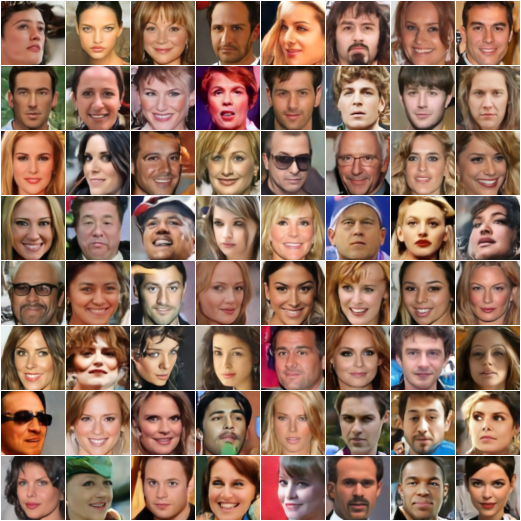}}
\caption{Random samples on CIFAR-10 ($\sigma=1.0$, $\tau=2.0$) and CelebA ($\sigma=2.0$, $\tau=8.0$).}
\label{samples}
\end{center}
\vskip -0.2in
\end{figure}

\begin{table}[t]
\caption{FID and Inception Scores on CIFAR-10.}
\label{scores-cifar}
\vskip 0.15in
\begin{center}
\begin{small}
\begin{sc}
\begin{tabular}{lcc}
\toprule
Models & FID & IS {}\\
\midrule
WGAN-GP &36.4&7.86$\pm$0.07\\
SN-SMMDGAN &25.0&7.3$\pm$0.1\\
SNGAN &21.7&8.22$\pm$0.05\\
NCSN &25.32&8.87$\pm$0.12\\
NCSNv2 &10.87&8.40$\pm$0.07\\
\midrule
\textbf{Ours} &{\bf 12.32}&{\bf 8.14$\pm$0.07}\\
\bottomrule
\end{tabular}
\end{sc}
\end{small}
\end{center}
\vskip -0.1in
\end{table}

\subsection{Image Interpolation and Inpainting with Stage 2}
To demonstrate usefulness of the proposed algorithm, we consider image interpolation and inpainting tasks. 

Interpolating images linearly in the data distribution $p_{\mathrm{data}}$ would induce artifacts. However, if we perturb the linear interpolation using a Gaussian noise with variance $\sigma^2$, and then use our stage 2 to denoise, we are able to obtain an interpolation without such artifacts. We find $\sigma^2 = 0.4$ is suitable for the image interpolation task for CelebA. Fig. \ref{interpolation} lists the image interpolation results. Our algorithm produces smooth image interpolation by gradually changing facial attributes.
\begin{figure}[ht]
\vskip 0.2in
\begin{center}
\centerline{\includegraphics[width=1.\columnwidth]{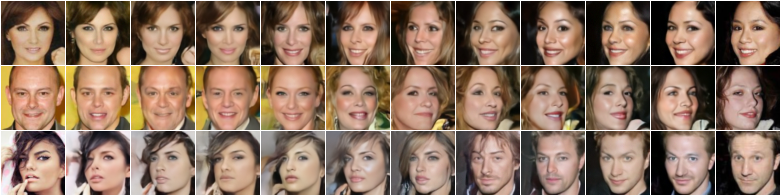}}
\caption{Image interpolation on CelebA. The first and last columns correspond to real images.}
\label{interpolation}
\end{center}
\vskip -0.2in
\end{figure}

The second stage can also be utilized for image inpainting with a little modification, inspired by the image inpainting algorithm with annealed Langevin dynamics in \cite{song2019generative}. Let $\mm$ be a mask with entries in $\{0, 1\}$ where $0$ corresponds to missing pixels. The idea for inpainting is very similar to interpolation. We treat $\mx \odot \mm + \sigma\meps$ as a sample from $q_{\sigma}$, where $\meps\sim\mathcal{N}(\mzero, \mI)$. Thus, we can use stage 2 to obtain samples from $p_{\mathrm{data}}$. The image inpainting procedure is given in algorithm \ref{alg:inpainting}, and the results are presented in Fig. \ref{inpainting}. Notice that we perturb $\my$ with $\sqrt{1-\frac{k+1}{N_2}}\sigma\mz$ at the end of each iteration. This is because the $k$-th iteration in stage 2 can be regarded as one-step Schr\"{o}dinger Bridge from $q_{\sqrt{1-k/N_2}\sigma}$ to $q_{\sqrt{1-{(k+1)}/N_2}\sigma}$. Thus, the particles are supposed to follow $q_{\sqrt{1-{(k+1)}/N_2}\sigma}(\mx)$ after the $k$-th iteration.
\begin{algorithm}[tb]
   \caption{Inpainting with stage 2}
   \label{alg:inpainting}
   \begin{algorithmic}
   \STATE {\bfseries Input:} $\my = \mx \odot \mm$, $\mm$
   \STATE Sample $\mz \sim \mathcal{N}(\mzero, \mI)$
   \STATE $\mx_0 = \my + \sigma\mz$
   \FOR{$k=0$ {\bfseries to} $N_2-1$}
   \STATE Sample $\meps_k \sim \mathcal{N}(\mzero, \mI)$
   \STATE $\mb(\mx_k)=\ms_\theta(\mx_k, \sqrt{1-\frac{k}{N_2}}\sigma)$
   \STATE $\mx_{k+1} = \mx_k + \frac{\sigma^2}{N_2}\mb(\mx_k)+\frac{\sigma}{\sqrt{N_2}}\meps_n$
   \STATE $\mx_{k+1} = \mx_{k+1} \odot (1-\mm) + (\my + \sqrt{1-\frac{k+1}{N_2}}\sigma\mz) \odot \mm$
   \ENDFOR
   \STATE {\bfseries return} $\mx_{N_2}$
   \end{algorithmic}
\end{algorithm}

\begin{figure}[ht]
\vskip 0.2in
\begin{center}
\centerline{\includegraphics[width=0.99\columnwidth]{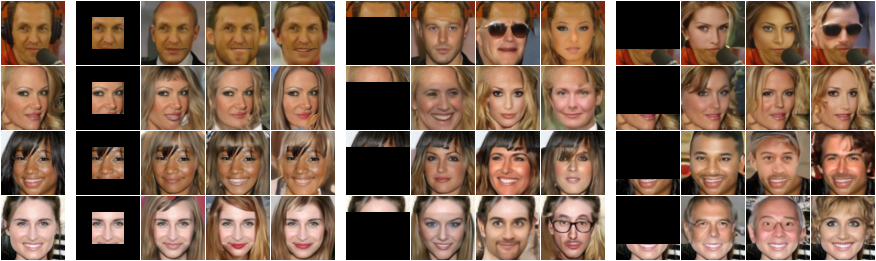}}
\caption{Image inpainting on CelebA. The leftmost column contains real images. Each occluded image is followed by three inpainting samples.}
\label{inpainting}
\end{center}
\vskip -0.2in
\end{figure}




%

\section{Conclusion}
We propose to  learn a generative model via   entropy interpolation with a Schr\"{o}dinger Bridge.
At the population level, this entropy interpolation   can be  characterized  via an SDE on $[0,1]$ with a time varying drift term.
 We derive a two-stage  Schr\"{o}dinger Bridge  algorithm by plugging the  drift term estimated  by a deep score estimator and a deep density estimator in the Euler-Maruyama method.
Under some smoothness assumptions of the target distribution,
we prove  the consistency
 of the proposed Schr\"{o}dinger Bridge approach, guaranteeing that the learned distribution converges to the target distribution.
Experimental results on multimodal synthetic data and benchmark data support our theoretical findings and demonstrate that the generative model via Schr\"{o}dinger Bridge  is comparable  with state-of-the-art  GANs, suggesting a new formulation of generative learning.

\section{Acknowledgement}
We thank the reviewers for their valuable comments. This work is supported in part by
the National Science Foundation of China under Grant 11871474 and by the research fund of
KLATASDSMOE,  the National Key Research and Development Program of China 208AAA0101100, Hong Kong Research Grant Council [16307818, 16301419, 16308120], Guangdong-Hong Kong-Macao Joint Laboratory project [2020B1212030001], Hong Kong Innovation and Technology Fund [PRP/029/19FX], Hong Kong University of Science and Technology (HKUST) [startup grant R9405, Z0428 from the Big Data Institute] and the HKUST-WeBank Joint Lab project. The computational task for this work was partially performed using the X-GPU cluster supported by the RGC Collaborative Research Fund: C6021-19EF.

\appendix

\section{Proofs}
\subsection{Proof of Theorem \ref{th01} }
\begin{thm}\label{th01}\cite{leonard2014survey}
If $\mu, \nu \ll \mathscr{L}$,  then SBP admits a unique solution $\mathbf{Q}^* = f^*(X_0)g^*(X_1)\mathbf{P}_{\tau}$, where
$f^*$, $g^*$ are $\mathscr{L}$-measurable nonnegative  functions satisfying the  Schr\"{o}dinger systems
$\left\{\begin{array}{l}
f^*(\mx) \mathbb{E}_{\mathbf{P}_{\tau}}\left[g^*\left(X_{1}\right) \mid X_{0}=\mx\right]= \frac{d \mu}{d\mathscr{L}}(\mx), \quad \mathscr{L}-a . e . \\
g^*(\my)  \mathbb{E}_{\mathbf{P}_{\tau}}\left[f^{*}\left(X_{0}\right) \mid X_{1}=\my\right]=\frac{d \nu}{d\mathscr{L}}(\my), \quad \mathscr{L}-a . e .
\end{array}\right.$
Furthermore, the pair $(\mathbf{Q}^*_{t},\mv^*_{t})$ with $$\mv^*_{t}(\mx) = \nabla_{\mx} \log \mathbb{E}_{\mathbf{P}_{\tau}}\left[g^{*}\left(X_{1}\right) \mid X_{t}= \mx\right]$$ solves the minimum action problem
$$\min_{\mu_t,\mv_t} \int_{0}^{1} \mathbb{E}_{\mz\sim \mu_t}[\|\mv_t(\mz)\|^2] d t$$ s.t.
$$\left\{\begin{array}{l}
\partial_{t}\mu_t = -\nabla \cdot(\mu_t \mv_t) +  \frac{\tau\Delta}{2} \mu_t, \quad \text { on }(0,1) \times \mathbb{R}^{d} \\
\mu_{0}=\mu, \mu_{1}=\nu.
\end{array}\right.
$$
\end{thm}
\begin{proof}
Theorem \ref{th01} follows from \cite{leonard2014survey}.
\end{proof}
\subsection{Proof of Theorem \ref{th02}}
\begin{thm}\label{th02}\cite{dai1991stochastic}
Let \begin{equation}\label{drift}
 \begin{aligned}
 &\mathbf{u}_{t}^* = \tau \mv^*_{t} = \tau \nabla_{\mx}\log g_t(\mx) \\
 = &\tau \nabla_{\mx}\log  \int h_{\tau}(t, \mx, 1, \my)g_1(\my) d \my.
 \end{aligned}
 \end{equation}
 Then,
$$\mathbf{u}^*_{t}(\mx)\in \arg\min_{\mathbf{u} \in \mathcal{U}}\mathbb{E}\left[\int_0^1\frac{1}{2\tau}\|\mathbf{u}_t\|^2\mathrm{d}t\right]$$
s.t.
\begin{equation}\label{sde}
\left\{\begin{array}{l}
\mathrm{d}\mx_t = \mathbf{u}_t \mathrm{d}t + \sqrt{\tau}\mathrm{d}\mw_t, \\
\mx_0\sim q(\mx),\quad \mx_1\sim p(\mx).
\end{array}\right.
\end{equation}
\end{thm}
\begin{proof}
Theorem \ref{th02} follows from \cite{dai1991stochastic}.
\end{proof}
\subsection{Proof of Theorem \ref{thm1}}
\begin{thm}
\label{thm1}
Define the density ratio $f(\mx)=\frac{q_{\sigma}(\mx)}{\Phi_{\sqrt{\tau}}(\mx)}$.
Then for the SDE
\begin{equation}
\label{stg1}
\mathrm{d}\mx_t = \tau\nabla\log\mathbb{E}_{\mz\sim\Phi_{\sqrt{\tau}}}[f(\mx_t+\sqrt{1-t}\mz)] \mathrm{d}t+\sqrt{\tau}\mathrm{d}\mw_t
\end{equation}
with initial condition $\mx_0 = \mzero$, we have $\mx_1 \sim q_{\sigma}(\mx)$.

And, for the SDE
\begin{equation}
\label{stg2}
\mathrm{d}\mx_t = \sigma^2\nabla\log q_{\sqrt{1-t}\sigma}(\mx_t)\mathrm{d}t+\sigma\mathrm{d}\mw_t
\end{equation}
with initial condition $\mx_0 \sim q_{\sigma}(\mx)$, we have $\mx_1 \sim p_{\mathrm{data}}(\mx)$.
\end{thm}
 \begin{proof}
 Denote
$$f_{0}(\mx) = f^*(\mx), \ \ g_{1}(\my) = g^*(\my),$$
$${f_{1}}(\my) = \mathbb{E}_{\mathbf{P}_{\tau}}\left[f^{*}\left(X_{0}\right) \mid X_{1}=\my\right] = \int h_{\tau}(0, \mx, 1, \my)f_{0}(\mx) d \mx,$$
$${g_{0}}(\mx)= \mathbb{E}_{\mathbf{P}_{\tau}}\left[g^*\left(X_{1}\right) \mid X_{0}=\mx\right] = \int h_{\tau}(0, \mx, 1, \my)g_{1}(\my) d \my.$$
 Then, the  Schr\"{o}dinger system  in Theorem \ref{th01} can also be characterized
  by
  \begin{equation}\label{sbs}
  q(\mx) = f_0(\mx) {g_{0}}(\mx), \ \  p(\my)=  {f_{1}}(\my)g_1(\my)
  \end{equation}
 For Eq. (\ref{stg1}), let $f_0(\mx) = \delta_{\mzero}(\mx)$ be the Dirac delta function,
 $f_1(\my) = \int h_{\tau}(0, \mx, 1, \my)f_0(\mx)\mathrm{d}\mx = \Phi_{\sqrt{\tau}}(\my),$
  $g_1(\mx) = \frac{q_{\sigma}(\mx)}{\Phi_{\sqrt{\tau}}(\mx)} = f(\mx)$,
 $g_0(\mzero) = \int h_{\tau}(0, \mzero, 1, \my)g_1(\my)\mathrm{d}\my = 1$.
 Then $f_i,g_i$, $i = 0,1$ solve  Schr\"{o}dinger system (\ref{sbs})   with $q = \delta_{\mzero}$, $p = q_{\sigma}$.
 Define
 \begin{align*}
 g_t(\mx) = &\int h_{\tau}(t, \mx, 1, \my)g_1(\my)\mathrm{d}\my = \mathbb{E}_{\my \sim \Phi_{\sqrt{(1-t)\tau}} }[f(\my)]\\
 &=  \sqrt{1-t}\mathbb{E}_{\mz\sim \Phi_{\sqrt{\tau}}} [f(\mx + \sqrt{1-t} \mz)].
 \end{align*}
 By Theorem \ref{th02}, $\mathbf{u}^{*}(\mx) = \tau \nabla_{\mx}\log  g_t(\mx) $ solves the optimal control problem
 $ \min_{\mathbf{u} \in \mathcal{U}}\mathbb{E}\left[\int_0^1\frac{1}{2\tau}\|\mathbf{u}_t\|^2\mathrm{d}t\right]$ such that
 \begin{equation*}
 \left\{\begin{array}{l}
 \mathrm{d}\mx_t = \mathbf{u}_t \mathrm{d}t + \sqrt{\tau}\mathrm{d}\mw_t \\
 \mx_0\sim \delta_{\mathbf{0}},\quad \mx_1\sim q_{\sigma}(\mx)
 \end{array}\right.
 \end{equation*}
 i.e., the dynamic of  Eq. (\ref{stg1}) will push $\delta_{\mathbf{0}}$ onto $q_{\sigma}$ from  $t=0$ to  $t=1$.

 For Eq. (\ref{stg2}),  
 let
  $f_0(\mx) = 1,$
 $$f_1(\my) = \int h_{\sigma^2}(0, \mx, 1, \my)f_0(\mx)\mathrm{d}\mx = 1,$$
 $g_1(\mx) = p_{\mathrm{data}}(\mx)$, $g_0(\mx) = \int h_{\sigma^2}(0, \mx, 1, \my)g_1(\my)\mathrm{d}\my = q_{\sigma}(\mx)$.
 Then, $f_i,g_i$, $i = 0,1$ solve Schr\"{o}dinger system (\ref{sbs})   with $q = q_{\sigma}$, $p = p_{\mathrm{data}}$ and $\tau = \sigma^2$.
 Define
 \begin{align*}
 g_t(\mx) = &\int h_{\sigma^2}(t, \mx, 1, \my)g_1(\my)\mathrm{d}\my = q_{\sqrt{1-t}\sigma}(\mx).
 \end{align*}
 By Theorem \ref{th02}, $\mathbf{u}^{*}(\mx) = \sigma^2 \nabla_{\mx}\log g_t(\mx) $ solves the optimal control problem
 $ \min_{\mathbf{u} \in \mathcal{U}}\mathbb{E}\left[\int_0^1\frac{1}{2\sigma^2}\|\mathbf{u}_t\|^2\mathrm{d}t\right]$ such that
 \begin{equation*}
 \left\{\begin{array}{l}
 \mathrm{d}\mx_t = \mathbf{u}_t \mathrm{d}t + \sigma\mathrm{d}\mw_t \\
 \mx_0\sim q_{\sigma}(\mx),\quad \mx_1\sim p_{\mathrm{data}}(\mx)
 \end{array}\right.
 \end{equation*}
 i.e., the dynamic of  Eq. (\ref{stg2}) will push $q_{\sigma}$ onto $p_{\mathrm{data}}$ from  $t=0$ to  $t=1$.
 \end{proof}
\subsection{Proof of Theorem \ref{thdes1}}
\begin{thm}\label{thdes1}
Assume that the  support of  $p_{\mathrm{data}}(\mx)$  is   contained in a compact set, and $f(\mx)$ is Lipschitz continuous and bounded.  Set the depth $\mathcal{D}$, width $\mathcal{W}$, and size $\mathcal{S}$ of  $\mathcal{NN}_{\phi} $
as $$\mathcal{D}=\mathcal{O} (\log (n)), \mathcal{W}= \mathcal{O}( n^{\frac{d}{2(2+d)}} / \log (n)),$$  $$\mathcal{S}=\mathcal{O}(n^{\frac{d-2 }{d+2 }} \log (n)^{-3}).$$ Then $\mathbb{E}[\|\hat{f}(\mx)-f(\mx)\|_{L^2(p_{\mathrm{data}})}]\rightarrow0$  as $n\rightarrow \infty.$
\end{thm}
\begin{proof}
Recall that \begin{equation}\label{edrift1}
\hat{f}(\mx) = \exp(\hat{r}_{\phi}(\mx)),
\end{equation}
where $\hat{r}_{\phi} \in \mathcal{NN}_{\phi}$ is the neural network that minimizes the empirical loss:
\begin{equation}\label{drest}
\begin{aligned}
& \hat{r}_{\phi}  \in {\arg\min}_{r_{\phi}\in \mathcal{NN}_{\phi}}\hat{L}(r_{\phi}), \,\text{where }\, \hat{L}(r_{\phi}) = \\
& \frac{1}{n}\sum_{i =1}^n[\log(1+\exp(-r_{\phi}(\widetilde{\mx}_i)))
 + \log(1+\exp(r_{\phi}(\mz_i)))],
\end{aligned}
\end{equation}

$\widetilde{\mx}_1,...,\widetilde{\mx}_n$ are i.i.d. samples from $q_{\sigma}(\mx)$,  $\mz_1,...,\mz_n$ are i.i.d. samples from $\Phi_{\sqrt{\tau}}(\mx)$.
Note that $f(\mx) = \exp{(r^*(\mx))}$ with
$$\begin{aligned}
r^{*}\in{\arg\min}_{r}\mathcal{L}(r),
\end{aligned}$$
where $\mathcal{L}(r)= \mathbb{E}_{q_{\sigma}(\mx)}\log (1+\exp(-r(\mx)))
 + \mathbb{E}_{\Phi_{\sqrt{\tau}}(\mx)}\log (1+\exp(r(\mx)))$.

Theorem \ref{thdes1} follows by showing $\|\hat{r}_{\phi}-r^*\|_{L^2(p_{\mathrm{data}})} \rightarrow 0$ as $n\rightarrow \infty$. By the assumption that $r^*(\mx)$ is Lipschitz continuous on a compact set and bounded, we use  $L_1$ and $B_1$ to denote its Lipschitz constant and the upper bound. Without loss of generality, we use $E= [-C, C]^d$ to denote its domain. By  Lemma \ref{appshen} (given in A.6) with $L = \log n$, $N = n^{\frac{d}{2(2+d)}}/\log n$, there exists a $\bar{r}_{\phi} \in \mathcal{NN}_{\phi}$ with
depth $\mathcal{D} = 12 \log n + 14+2d,$ width $\mathcal{W} = 3^{d+3}\max\{d(n^{\frac{d}{2(2+d)}}/\log n)^{\frac{1}{d}},n^{\frac{d}{2(2+d)}}/\log n+1\},$ and size $\mathcal{S} =n^{\frac{d-2}{d+2}}/(\log^4 n),$ $\mathcal{B} = 2B_1$,
such that
\begin{equation}\label{app1}
\|\bar{r}_{\phi} - r^*\|_{L^2(p_{\mathrm{data}})} \leq  38 L_1C\sqrt{d}  n^{-\frac{1}{d+2}}.
\end{equation}
Using Taylor expansion and the boundness of $r_{\phi} \in \mathcal{NN}_{\phi}$ and $r^*$, it is easy to show that $\mathcal{L}({r}_{\phi})-\mathcal{L}(r^*)$ is sandwiched by
 $\|\bar{r}_{\phi} - r^*\|_{L^2(p_{\mathrm{data}})}^2 $,    i.e., $\forall r_{\phi} \in \mathcal{NN}_{\phi}$
\begin{equation}\label{lub}
\begin{aligned}
C_{1,\mathcal{B}} \|r_{\phi} - r^*\|_{L^2(p_{\mathrm{data}})}^2 \leq \mathcal{L}(r_{\phi})-\mathcal{L}(r^*) \\ \leq C_{2,\mathcal{B}} \|r_{\phi} - r^*\|_{L^2(p_{\mathrm{data}})}^2.
\end{aligned}
\end{equation}
Then,
\begin{align}
&C_{1,\mathcal{B}}\|\hat{r}_{\phi} - r^{*}\|_{L^2}^2 \leq \mathcal{L}(\hat{r}_{\phi})-\mathcal{L}(r^{*}) \nonumber\\
= &\mathcal{L}(\hat{r}_{\phi}) - \hat{\mathcal{L}}(\hat{r}_{\phi}) + \hat{\mathcal{L}}(\hat{r}_{\phi})-  \hat{\mathcal{L}}(\bar{r}_{\phi})\nonumber\\
& + \hat{\mathcal{L}}(\bar{r}_{\phi}) - \mathcal{L}(\bar{r}_{\phi}) +\mathcal{L}(\bar{r}_{\phi})  - \mathcal{L}(r^*)\nonumber\\
\leq & 2 \sup_{r\in \mathcal{NN}_{\phi}} |\mathcal{L}(r) - \hat{\mathcal{L}}(r) |+ C_{2,\mathcal{B}}\|\bar{r}_{\phi} - r^{*}\|_{L^2(\nu)}^2\nonumber\\
\leq & 2 \sup_{r\in \mathcal{NN}_{\phi}} |\mathcal{L}(r) - \hat{\mathcal{L}}(r) |+ 38C_{2,\mathcal{B}} L_1C\sqrt{d}  n^{-\frac{1}{d+2}}, \label{A7}
\end{align}
where we use the definition  of $\hat{r}_{\phi}$, $r^*$, and $\bar{r}_{\phi}$ as well as  (\ref{app1}) and (\ref{lub}).
  Next, we finish the proof by bounding the empirical process term in (\ref{A7}).
  Let $\mathbf{O} = (\tilde{\mx},\mz)$ be the  random variable pair, with $\mx \sim p_{\mathrm{data}}$, $\mz \sim \Phi_{\sqrt{\tau}}$, and  $\{\mathbf{O}_{i}\}_{i=1}^{n}$ be $n$ i.i.d. copies of  $\mathbf{O}$.  Denote
  $\mathbf{o} = (\tilde{x},z) \in \mathbb{R}^d \times \mathbb{R}^d$ be a realization of $\mathbf{O}$, and
  define  $$b(r,\mathbf{o}) = \log (1+\exp^{-r(\tilde{x})})+\log (1+\exp^{r(z)}).$$
  It is easy to check that $b(r,\mathbf{o})$ is 1-Lipschitz on $r$, i.e.,
  \begin{equation}\label{lip}
|b(r,\mathbf{o})-b(\tilde{r},\mathbf{o})|\leq |r(\tilde{x}) - \tilde{r}(\tilde{x})|+|r({z}) - \tilde{r}({z})|
  \end{equation}
  Let $\widetilde{\mathbf{O}}_i$ be a ghost i.i.d. copy of $\mathbf{O}_{i},$ and $\sigma_i (\epsilon_i) $ be the i.i.d. Rademacher random (standard  normal) variables that are independent with
  $\widetilde{\mathbf{O}}_i$ and $\mathbf{O}_{i}$,  $i = 1,...n.$
  We need the following
  results \eqref{A8}-\eqref{A9} to upper bound the expected value of the right hand side term in \eqref{A7}.
 \begin{equation}\label{A8}
 \mathbb{E}_{\{\mathbf{O}_i\}_{i=1}^n} [\sup_{r} |\mathcal{L}(r) - \hat{\mathcal{L}}(r) | ] \leq \mathcal{O}(\mathcal{G}(\mathcal{NN})),
 \end{equation}
  where $\mathcal{G}(\mathcal{NN})$ is the Gaussian complexity \cite{bartlett2002rademacher} of $\mathcal{NN}_{\phi}$ defined as
  $$\mathcal{G}(\mathcal{NN}) =  \mathbb{E}_{\{\mathbf{O}_i, \epsilon_i \}_{i}^n}[\sup_{r\in \mathcal{NN}_{\phi}}|\frac{1}{n}\sum_{i=1}^n\epsilon_i b(r,\mathbf{O}_i)|].$$
  \textbf{Proof of} \eqref{A8}.\\
  Obviously,
   $$\mathcal{L}(r) = \mathbb{E}_{\mathbf{O}} [b(r,\mathbf{O})] = \frac{1}{n}\mathbb{E}_{\widetilde{\mathbf{O}}_i} [b(r,\widetilde{\mathbf{O}}_i],$$ and
   $$ \widehat{\mathcal{L}}(r) = \frac{1}{n}\sum_{i=1}^n b(r,\mathbf{O}_i).$$
      Let $$\mathcal{R}(\mathcal{NN}) = \frac{1}{n} \mathbb{E}_{\{\mathbf{O}_i, \sigma_i\}_{i}^n}[\sup_{r\in \mathcal{NN}_{\phi}}|\sum_{i=1}^n\sigma_i b(r,\mathbf{O}_i)|]$$ be the 	
Rademacher complexity of $\mathcal{NN}_{\phi}$ \cite{bartlett2002rademacher}.
    Then,
  \begin{align*}
  &\mathbb{E}_{\{\mathbf{O}_i\}_{i=1}^n} [\sup_{r} |\mathcal{L}(r) - \hat{\mathcal{L}}(r) | ] \\
  =&\frac{1}{n} \mathbb{E}_{\{\mathbf{O}_i\}_{i}^n} [\sup_{r} |\sum_{i=1}^n (\mathbb{E}_{\widetilde{\mathbf{O}}_i} [b(r,\widetilde{\mathbf{O}}_i)] - b(r,\mathbf{O}_i))|]\\
  \leq & \frac{1}{n} \mathbb{E}_{\{\mathbf{O}_i, \widetilde{\mathbf{O}}_i\}_{i}^n} [\sup_{r} |b(r,\widetilde{\mathbf{O}}_i) - b(r,{\mathbf{O}}_i)|]\\
  = & \frac{1}{n} \mathbb{E}_{\{\mathbf{O}_i, \widetilde{\mathbf{O}}_i,\sigma_i \}_{i}^n} [\sup_{r } |\sum_{i=1}^n\sigma_i(b(r,\widetilde{\mathbf{O}}_i) - b(r,{\mathbf{O}}_i))|]\\
  \leq & \frac{1}{n}  \mathbb{E}_{\{\mathbf{O}_i, \sigma_i \}_{i}^n} [\sup_{r } |\sum_{i=1}^n\sigma_i b(r,{\mathbf{O}}_i)| ]\\
  & + \frac{1}{n}  \mathbb{E}_{\{\widetilde{\mathbf{O}}_i,\sigma_i \}_{i}^n} [\sup_{r} |\sum_{i=1}^n\sigma_i b(r,\widetilde{\mathbf{O}}_i)| ] \\
  = & 2\mathcal{R}(b\circ\mathcal{NN})\\
  \leq & 4\mathcal{R}(\mathcal{NN})\\
  \leq & \mathcal{O}(\mathcal{G}(\mathcal{NN})),
  \end{align*}
  where the first inequality follows from Jensen's inequality, and the second equality holds since  both $\sigma_i(b(r,\widetilde{\mathbf{O}}_i) - b(r,{\mathbf{O}}_i))$ and $b(r,\widetilde{\mathbf{O}}_i) - b(D,{\mathbf{O}}_i)$ are governed by the same law, and the last  equality holds since the distribution of the two terms are the same. In the  third inequality, we use the   Lipschitz contraction property of Rademacher complexity, see Theorem 12 in \cite{bartlett2002rademacher},  and \eqref{lip}. The last inequality holds since the  relationship between the Gaussian complexity and  the Rademacher complexity, see for Lemma 4 in \cite{bartlett2002rademacher}.

  Next, we bound the Gaussian complexity.
  \begin{equation}\label{A9}
  \begin{aligned}
  &\mathcal{G}(\mathcal{NN}) \leq\\
  &\mathcal{O}(\mathcal{B} \sqrt{\frac{n}{\mathcal{D}\mathcal{S}\log \mathcal{S}}}\log \frac{n}{\mathcal{D}\mathcal{S}\log \mathcal{S}} \exp(-\log^2 \frac{n}{\mathcal{D}\mathcal{S}\log \mathcal{S}})).
  \end{aligned}
  \end{equation}
  \textbf{Proof of \eqref{A9}}.\\
  Since $\mathcal{NN}_{\phi}$ is  closed under negation,
\begin{align*}
&\mathcal{G}(\mathcal{NN}) =  \mathbb{E}_{\{\mathbf{O}_i, \epsilon_i \}_{i}^n}[\sup_{r\in \mathcal{NN}_{\phi}}\frac{1}{n}\sum_{i=1}^n\epsilon_i b(r,{\mathbf{O}}_i)]\\
& = \mathbb{E}_{\mathbf{O}_i}[ \mathbb{E}_{\epsilon_i}[\sup_{r\in \mathcal{NN}_{\phi}}\frac{1}{n}\sum_{i=1}^n\epsilon_i b(r,{\mathbf{O}}_i)]|\{\mathbf{O}_i\}_{i=1}^n].
\end{align*}
Conditioning on $\{\mathbf{O}_i\}_{i =1}^n$,
$\forall r, \tilde{r} \in \mathcal{NN}_{\phi}$, it easy to check $$\mathbb{V}_{\epsilon_i} [\frac{1}{n}\sum_{i=1}^n\epsilon_i (b(r,{\mathbf{O}}_i) - b(\tilde{r},{\mathbf{O}}_i))] = \frac{d_{\mathcal{NN}}(r,\tilde{r})}{\sqrt{n}},$$
where $$d_{\mathcal{NN}}(r,\tilde{r}) = \frac{1}{\sqrt{n}} \sqrt{\sum_{i =1}^n (b(r,{\mathbf{O}}_i) - b(\tilde{r},{\mathbf{O}}_i))^2}.$$
Denote  $\mathfrak{C}(\mathcal{NN},d_{\mathcal{NN}},\delta)$ as the covering number of $\mathcal{NN}_{\phi}$
under the metric $d_{\mathcal{NN}}$ with radius $\delta$, and let $\mathrm{Pdim}_{\mathcal{NN}}$ be the Pseudo-dimension of $\mathcal{NN}_{\phi}$.
Since the diameter of $\mathcal{NN}_{\phi}$ under $d_{\mathcal{NN}} $ is at most $\mathcal{B}$,  we have
  \begin{align*}
  &\mathcal{G}(\mathcal{NN})\\
  \leq & \frac{c}{\sqrt{n}} \mathbb{E}_{\{\mathbf{O}_i\}_{i=1}^n}[\int_{0}^{B} \sqrt{\log \mathfrak{C}(\mathcal{NN},d_{\mathcal{NN}},\delta)} \mathrm{d} \delta]\\
  \leq&\frac{c}{\sqrt{n}} \mathbb{E}_{\{\mathbf{O}_i\}_{i=1}^n}[\int_{0}^{\mathcal{B}} \sqrt{\log  \mathfrak{C}(\mathcal{NN},d_{\mathcal{NN},\infty},\delta)} \mathrm{d} \delta]\\
  \leq&\frac{c}{\sqrt{n}}\int_{0}^{\mathcal{B}} \sqrt{\mathrm{Pdim}_{\mathcal{NN}} \log \frac{2e\mathcal{B}n}{\delta \mathrm{Pdim}_{\mathcal{NN}} }} \mathrm{d} \delta, \\
  \leq& c \mathcal{B}(\frac{n}{\mathrm{Pdim}_{\mathcal{NN}}})^{1/2}\log (\frac{n}{\mathrm{Pdim}_{\mathcal{NN}}}) \exp( -\log^2(\frac{n}{\mathrm{Pdim}_{\mathcal{NN}}}))\\
  \leq& c\mathcal{B} \sqrt{\frac{n}{\mathcal{D}\mathcal{S}\log \mathcal{S}}}\log \frac{n}{\mathcal{D}\mathcal{S}\log \mathcal{S}} \exp(-\log^2 \frac{n}{\mathcal{D}\mathcal{S}\log \mathcal{S}}),
  \end{align*}
  where $c$ is a constant which may vary on different places, the first inequality follows from the chaining  Theorem 8.1.3 in \cite{vershynin2018high}, the second inequality holds due to
  $\mathfrak{C}(\mathcal{NN},d_{\mathcal{NN}},\delta)\leq \mathfrak{C}(\mathcal{NN},d_{\mathcal{NN},\infty},\delta)$, in the third inequality we use
  the relationship between the metric entropy and the Pseudo-dimension of the ReLU networks  $\mathcal{NN}_{\phi}$ \cite{anthony2009neural}, i.e.,
  $$\log \mathfrak{C}(\mathcal{NN},d_{\mathcal{NN},\infty},\delta)) \leq \mathrm{Pdim}_{\mathcal{NN}} \log \frac{2e\mathcal{B}n}{\delta\mathrm{Pdim}_{\mathcal{NN}}},$$
  the fourth inequality follows by  some calculation,
  and the last inequality  holds due to the  upper bound of Pseudo-dimension for the ReLU network $\mathcal{NN}_{\psi}$ satisfying  $$\mathrm{Pdim}_{\mathcal{NN}} =\mathcal{O} (\mathcal{D}\mathcal{S}\log \mathcal{S}),$$ see \cite{bartlett2019}.

  Finally, by \eqref{A7}-\eqref{A9} and the choice of $\mathcal{D}$, $\mathcal{W}$ and $\mathcal{S}$,
  we get $\mathbb{E}[\|\hat{r}_{\phi} -r^*\|_{L^2}^2] \leq \mathcal{O}(n^{-\frac{2}{2+d}})\rightarrow 0$ as $n\rightarrow \infty.$
\end{proof}
\subsection{Proof of Theorem \ref{thdes2}}
\begin{thm}\label{thdes2}
Assume that  $p_{\mathrm{data}}(\mx)$ is differentiable with  bounded support, and $ \nabla\log q_{\tilde{\sigma}}(\mx)$ is Lipschitz continuous and bounded for $(\tilde{\sigma},\mx)\in [0,\sigma] \times \mathbb{R}^d$.  Set the depth $\mathcal{D}$, width $\mathcal{W}$, and size $\mathcal{S}$ of  $\mathcal{NN}_{\theta} $
as $$\mathcal{D}=\mathcal{O} (\log (n)), \mathcal{W}= \mathcal{O}( \max\{n^{\frac{d}{2(2+d)}} / \log (n), d\}),$$  $$\mathcal{S}=\mathcal{O}(d n^{\frac{d-2 }{d+2 }} \log (n)^{-3}).$$ Then
$\mathbb{E}[\| \|\widehat{\nabla\log q_{\tilde{\sigma}}}(\mx)-\nabla\log q_{\tilde{\sigma}}(\mx)\|_2\|_{L^2(q_{\tilde{\sigma}})}]\rightarrow0$  as $m,n\rightarrow \infty.$
\end{thm}
\begin{proof}
We give the proof for the fixed   $\tilde{\sigma}$ case. The case that $\tilde{\sigma}$ vary in a interval can be treated similarly.
Recall that
$$\ms^*\in{\arg\min}_{\ms} \frac12\mathbb{E}_{\mx\sim q_{\tilde{\sigma}}(\mx)}\|\ms(\mx)-\nabla_{\mx}\log q_{\tilde{\sigma}}(\mx)\|^2$$
is equivalent to $\ms^* \in {\arg\min}_{\ms}\mathcal{L}(\ms)$,
\begin{align*}
\text{where}\,\mathcal{L}(\ms)=\frac12\mathbb{E}_{\mx\sim p_{\mathrm{data}}(\mx),\mz\sim \mathscr{N}(\mathbf{0},\tilde{\sigma}^2\mathbf{I}) }\left\|\ms(\mx+\mz)+\frac{\mz}{\tilde{\sigma}^2}\right\|^2.
\end{align*}
Since $\widehat{\nabla\log q_{\tilde{\sigma}}}(\mx) = \hat{\ms}_{\theta}(\mx;\tilde{\sigma})$ (we use $\hat{\ms}_{\theta}(\mx)$ to denote $\hat{\ms}_{\theta}(\mx;\tilde{\sigma})$ for short), where
$$\hat{\ms}_{\theta} \in {\arg\min}_{\ms_{\theta} \in \mathcal{NN}_{\theta}} \hat{{\mathcal{L}}}(\ms_{\theta}),$$
$\hat{{\mathcal{L}}}(\ms_{\theta}) = \sum_{i=1}^n \left\|\ms_{\theta}(\mx_i+\mz_i)+\frac{\mz_i}{\tilde{\sigma}^2_j}\right\|^{2}/(2n)$,
$\mx_i$ are i.i.d. samples from $p_{\mathrm{data}}$, and $\mz_i$ are i.i.d. samples from $\Phi_{\tilde{\sigma}},$ $i = 1,...,n$.
What we need to prove is
\begin{align*}
&\mathbb{E}_{\mx_i,\mz_i}[\| \|\ms^*-\hat{\ms}_{\theta}\|_2\|_{L^2(q_{\tilde{\sigma}})}^2]\\
= &\mathbb{E}_{\mx_i,\mz_i}[\mathbb{E}_{\mx\sim q_{\tilde{\sigma}}}[ \|\ms^*(\mx)-\hat{\ms}_{\theta}(\mx)\|^2]]\rightarrow0
\end{align*}
as $n\rightarrow \infty.$
Since the functional  $\mathcal{L}$ and $\hat{\mathcal{L}}$ are both quadratic, it is easy to conclude that
\begin{align}\label{th27}
&\mathbb{E}_{\mx\sim q_{\tilde{\sigma}}}[ \|\ms^*(\mx)-\hat{\ms}_{\theta}(\mx)\|^2]\nonumber\\
= & \mathbb{E}_{\mx\sim p_{\mathrm{data}},\mz\sim\Phi_{\tilde{\sigma}}}[ \|\ms^*(\mx+\mz)-\hat{\ms}_{\theta}(\mx+\mz)\|^2]\nonumber \\
= & \mathcal{L}(\hat{\ms}_{\theta}) - \mathcal{L}(\ms^*)\nonumber\\
= & \mathcal{L}(\hat{\ms}_{\theta}) - \hat{\mathcal{L}}(\hat{\ms}_{\theta}) + \hat{\mathcal{L}}(\hat{\ms}_{\theta})-  \hat{\mathcal{L}}(\bar{\ms}_{\theta})\nonumber\\
& +   \hat{\mathcal{L}}(\bar{\ms}_{\theta}) - \mathcal{L}(\bar{\ms}_{\theta}) +\mathcal{L}(\bar{\ms}_{\theta})  - \mathcal{L}(\ms^*)\nonumber\\
\leq & 2 \sup_{\ms \in \mathcal{NN}_{\theta}} |\mathcal{L}(\ms) - \hat{\mathcal{L}}(\ms) |+ \mathbb{E}_{\mx\sim q_{\sigma}}[\|\bar{\ms}_{\theta}(\mx) - \ms^{*}\|_2^2]\nonumber\\
\leq & 2 \sup_{\ms\in \mathcal{NN}_{\theta}} |\mathcal{L}(\ms) - \hat{\mathcal{L}}(\ms) |+ \inf_{\bar{\ms}\in \mathcal{NN}_{\theta}}\mathbb{E}_{\mx\sim q_{\sigma}}[\|\bar{\ms}_{\theta}(\mx) - \ms^{*}\|_2^2],
\end{align}
where we use $\hat{\ms}$ as a minimizer and $\bar{\ms}$ as an arbitrary element of $\mathcal{NN}_{\theta}$ in the first inequality,  and  we take infimum over $\bar{\ms}\in \mathcal{NN}_{\theta}$ in the second inequality. We need to bound the two terms on the  right hand side of
(\ref{th27}).
The terms $\inf_{\bar{\ms}\in \mathcal{NN}_{\theta}}\mathbb{E}_{\mx\sim q_{\sigma}}[\|\bar{\ms}_{\theta}(\mx) - \ms^{*}\|_2^2]$ and
$\sup_{\ms\in \mathcal{NN}_{\theta}} |\mathcal{L}(\ms) - \hat{\mathcal{L}}(\ms) |$ are the so called
 approximation error and statistical error. They can be bounded by using the similar technique when we prove (\ref{app1}) and (\ref{A7}), respectively. Here we directly give the bounds and omit the details.
 By setting,
 $$\mathcal{D}=\mathcal{O} (\log (n)), \mathcal{W}= \mathcal{O}( n^{\frac{d}{2(2+d)}} / \log (n)),$$  $$\mathcal{S}=\mathcal{O}(n^{\frac{d-2 }{d+2 }} \log (n)^{-3}).$$ Then $$\inf_{\bar{\ms}\in \mathcal{NN}_{\theta}}\mathbb{E}_{\mx\sim q_{\sigma}}[\|\bar{\ms}_{\theta}(\mx) - \ms^{*}\|_2^2]\leq \mathcal{O}( dn^{-\frac{2}{d+2}}),$$
 $$\sup_{\ms\in \mathcal{NN}_{\theta}} |\mathcal{L}(\ms) - \hat{\mathcal{L}}(\ms) |\leq \mathcal{O}(n^{-\frac{2}{d+2}}).$$
  Thus, Theorem  \ref{thdes2} follows by plugging these above two displays into (\ref{th27}) and setting $n\rightarrow \infty.$
\end{proof}

\begin{thm}\label{thm2}
Under Assumptions 1-4,
$$\mathbb{E}[\mathcal{W}_2(\mathrm{Law}(\mx_{N_2}), p_{\mathrm{data}})] \rightarrow 0, \ \ \mathrm{as} \ \  n,N_1,N_2,N_3 \rightarrow \infty,$$
where $\mathcal{W}_2$ is the 2-Wasserstein distance between two distributions.
\end{thm}
\begin{proof}
Recall that  $$ D_1(t,\mx) =  \nabla\log\mathbb{E}_{\mz\sim\Phi_{\sqrt{\tau}}}[f(\mx+\sqrt{1-t}\mz)],$$
$$ D_2(t,\mx) =  \nabla\log q_{\sqrt{1-t}\sigma}(\mx),$$ and $$h_{\sigma,\tau}(\mx_1,\mx_2) = \exp{\left(\frac{\|\mx_1\|^2}{2\tau}\right)}p_{\mathrm{data}}(\mx_1 + \sigma \mx_2).$$
Some calculation shows
 \begin{align}\label{drift1}
& D_1(t,\mx) = \nabla\log\mathbb{E}_{\mz\sim\Phi_{\sqrt{\tau}}}[f(\mx+\sqrt{1-t}\mz)] \nonumber \\
= & \frac{\mathbb{E}_{\mz\sim\Phi_{\sqrt{\tau}}}\left[f(\mx+\sqrt{1-t}\mz)\nabla\log f(\mx+\sqrt{1-t}\mz)\right]}{\mathbb{E}_{\mz\sim\Phi_{\sqrt{\tau}}}[f(\mx+\sqrt{1-t}\mz)]},
\end{align}
and
\begin{equation}\label{drift2}
\nabla\log f(\mx) = \nabla\log q_{\sigma}(\mx) + \mx / \tau.
\end{equation}
 Let $\hat{D}_{1}(t,\mx) $ be an estimated version of $D_1(t,\mx)$ by replacing $f(\mx)$ and $\nabla\log f(\mx) = \nabla\log q_{\sigma}(\mx) +\mx / \tau$ with $\hat{f}(\mx)$ and $\hat{\ms}_{\theta}(\mx;\sigma)+\mx / \tau$, respectively.
 By Theorem  \ref{thdes1} and  \ref{thdes2}, we know that $$\hat{D}_{1}(t,\mx)\rightarrow{D}_{1}(t,\mx)\ \ \mathrm{as}\ \  n\rightarrow\infty.$$
Similarly, we know that $$\hat{D}_{2}(t,\mx) = \hat{\ms}_{\theta}(\mx;\sqrt{1-t}\sigma)   \rightarrow{D}_{2}(t,\mx), \ \ \mathrm{as} \ \ n\rightarrow\infty.$$
Recall that the iteration of state 1 in our Schr{\"o}dinger Bridge algorithm reads
\begin{equation}\label{emn1}
\begin{aligned}
 &\mx_{k+1} = \mx_k + \frac{\tau}{N_1}\mb(t_k,\mx_k)+\sqrt{\frac{\tau}{N_1}}\meps_k,\\
 &\mx_0 = \mathbf{0}, \ \ k = 0,...N_1-1,
\end{aligned}
\end{equation}
   where
    $$\mb(t_k,\mx_k)=\frac{\sum_{i=1}^{N_3}\hat{f}(\tilde{\mx}_i)[\hat{\ms}_{\theta}(\tilde{\mx}_i,\sigma)+\sqrt{\left(1-t_k\right)/\tau}\mz_i]}{\sum_{i=N_3+1}^{2N_3} \hat{f}(\tilde{\mx}_i)} + \frac{\mx_k}{\tau},$$
     $\tilde{\mx}_i = \mx_k+\sqrt{\tau\left(1-t_k\right)}\mz_i$, $i = 1,..., 2N_3$, $t_k = \frac{k}{N_1}$,
  $\{\mz_i\}_{i=1}^{2N_3}$, and $  \meps_k \sim \mathscr{N}(\mathbf{0}, \mI)$.
  Note that $\mb(t,\mx)$ is a Monte Carlo version of $\hat{D}_1(t,\mx)$ and converges to it as the number of samples $N_3\rightarrow \infty$.
 Then, $\forall (t,\mx)$
 \begin{equation}\label{cod}
 \mb(t,\mx)\rightarrow D_1(t,\mx)\ \ \mathrm{as}\ \  n, N_3\rightarrow\infty.
 \end{equation}
 By Assumption 1 and Assumption 4, we can show that the above consistency results hold uniformly for $(t,\mx)\in [0,1]\times \mathrm{supp}(p_{\mathrm{data}})$.
 The Euler-Maruyama
method for solving for SDE (\ref{stg1}) with step size $s = 1/N_1$, $t_k = k/N_1$
reads
\begin{equation}\label{emex1}
\begin{aligned}
&X_{k+1} = X_k + \frac{\tau}{N_1}D_1(t_k,X_k)+\sqrt{\frac{\tau}{N_1}}\meps_k,\\
&X_0 = \mathbf{0},\ \  k=0,...,N_1-1.
\end{aligned}
\end{equation}
Under our Assumptions 2 and 3, SDE (\ref{stg1}) admits a strong solution and \eqref{au1}-\eqref{au2} in Lemma \ref{aulem}  hold (see A.6). By the classical theory of Euler-Maruyama methods for solving SDEs  \cite{higham2001algorithmic},
$$\mathcal{W}_2(\mathrm{Law}(X_{N_1}), q_{\sigma}) =\mathcal{O}(1/\sqrt{N_1}) \rightarrow 0\ \ \mathrm{as}\ \  N_1 \rightarrow \infty.$$
Using the triangle inequality,  we prove
\begin{equation}\label{css1}
\mathcal{W}_2(\mathrm{Law}(\mx_{N_1}), q_{\sigma}) \rightarrow 0\ \ \mathrm{as}\ \ n,N_3, N_1 \rightarrow \infty,
\end{equation}
by showing $$\mathcal{W}_2(\mathrm{Law}(\mx_{N_1}),\mathrm{Law}(X_{N_1}))\rightarrow 0\ \ \mathrm{as}\ \ n, N_3 \rightarrow \infty.$$
Recall the definition of $\mx_{k}$ in (\ref{emn1}) and $X_{k}$ in (\ref{emex1}). We have
\begin{align*}
&\|\mx_{k}-X_{k}\|_2^2\\
\leq & \|\mx_{{k-1}}-X_{k-1}\|_2^2\\
&+\left(\frac{\tau}{N_1}\|D_1(t_{k-1},X_{k-1})-b(t_{k-1},\mx_{k-1})\|_2\mathrm{d} \ell \right)^2\\
&+2\frac{\tau}{N_1}\|\mx_{k-1}-X_{k-1}\|_2
\|D_1(t_{k-1},X_{k-1})-b(t_{k-1},\mx_{{k-1}})\|_2\\
\leq & (1+\tau/N_1) \|X_{{k-1}}-\mx_{{k-1}}\|_2^2\\
&+(\tau/N_1+ \tau^2/N_1^2)\|D_1(t_{k-1},X_{k-1})-b(t_{k-1},\mx_{{k-1}})\|_2^2\\
\leq & (1+\tau/N_1) \|X_{{k-1}}-\mx_{{k-1}}\|_2^2\\
&+2(\tau/N_1+ \tau^2/N_1^2)\|D_1(t_{k-1},X_{k-1})-D_1(t_{k-1},\mx_{{k-1}})\|_2^2\\
&+2(\tau/N_1+ \tau^2/N_1^2)\|D_1(t_{k-1},\mx_{k-1})-b(t_{k-1},\mx_{{k-1}})\|_2^2\\
\leq & (1+\tau/N_1) \|X_{{k-1}}-\mx_{{k-1}}\|_2^2\\
&+2C_2(\tau/N_1+ \tau^2/N_1^2)\|X_{k-1}-\mx_{{k-1}}\|_2^2\\
&+2(\tau/N_1+ \tau^2/N_1^2)o(1)\\
=&(1+\tau/N_1+2C_2(\tau/N_1+ \tau^2/N_1^2)) \|X_{{k-1}}-\mx_{{k-1}}\|_2^2\\
&+2(\tau/N_1+ \tau^2/N_1^2)o(1).
\end{align*}
where
the fourth  inequality holds by Assumption 3 and (\ref{cod}).
Taking expectation on the above display, we get
\begin{align*}
&\mathbb{E}[\|\mx_{k}-X_{k}\|_2^2]\\
\leq& (1+\tau/N_1+2C_2(\tau/N_1+ \tau^2/N_1^2)) \mathbb{E}[\|X_{{k-1}}-\mx_{{k-1}}\|_2^2]\\
&+2(\tau/N_1+ \tau^2/N_1^2)o(1).
\end{align*}
From the above display and the fact that $\mx_{0}=X_{0}=\mathbf{0}$, we can conclude that
\begin{align*}
&\mathbb{E}[\|\mx_{k}-X_{k}\|_2^2]\\
\leq&2(k-1)(\tau/N_1+ \tau^2/N_1^2)o(1)\leq 2(\tau+ \tau^2/N_1)o(1),\\
&\forall \ \ 1\leq k\leq N_1.
\end{align*}
Thus, we have
\begin{equation}\label{w2}
\mathcal{W}_2(\mathrm{Law}(X_{N_1}),\mathrm{Law}(\mx_{N_1}))\rightarrow 0, \ \ \mathrm{as}\ \  n, N_3 \rightarrow \infty.
\end{equation}
The consistency results \eqref{css1} for the first stage in Schr\"{o}dinger Bridge algorithm has been established.
For the second stage, the iteration reads
\begin{align*}
& \mx_{k+1} = \mx_k + \frac{\sigma^2}{N_2}\mb(\mx_k)+\frac{\sigma}{\sqrt{N_2}}\meps_k,\\
& k =0, ..., N_2-1, \mx_0 =\mx_{N_1},
\end{align*}
where $\mb(\mx_k)=\hat{\ms}_\theta(\mx_k, \sqrt{1-\frac{k}{N_2}}\sigma)$ and $\meps_k \sim \mathscr{N}(\mathbf{0}, \mI)$.
The Euler-Maruyama
method for solving for SDE (\ref{stg2}) with step size $s = 1/N_2$, $t_k = k/N_2$
reads
\begin{align*}
& X_{k+1} = X_k + \frac{\sigma^2}{N_2}D_2(t_k,X_k)+\sqrt{\frac{\sigma}{N_2}}\meps_k,\\
& X_0 \sim q_{\sigma},\ \  k=0,...,N_2-1.
\end{align*}
Then, the consistency results of the second stage can be proved similarly by repeating the part between Equation (\ref{emex1}) and Equation (\ref{w2}) and using the consistency results of the first stage, we omit the details here.
\end{proof}

\subsection{Additional Lemmas}
\begin{lemma}\label{appshen}
 Let $f$ be a uniformly continuous function defined on $E \subseteq[-R, R]^{d}$. For arbitrary $L \in \mathbb{N}^{+}$ and $N \in \mathbb{N}^{+},$ there exists a function  ReLU network $f_{\phi}$ with width $3^{d+3} \max \left\{d\left\lfloor N^{1 / d}\right\rfloor, N+1\right\}$ and depth $12 L+14+2 d$ such that
$$
\|f-f_{\phi}\|_{L^{\infty}(E)} \leq 19 \sqrt{d} \omega_{f}^{E}\left(2 R N^{-2 / d} L^{-2 / d}\right),
$$
where, $\omega_{f}^{E}(t)$  is the  modulus of continuity of $f$ satisfying $\omega_{f}^{E}(t)\rightarrow 0$ as $t\rightarrow 0^+.$
\end{lemma}
\begin{proof}
This is  Theorem 4.3 in  \cite{shen2019deep}.
\end{proof}

\begin{lemma}\label{aulem}
Let $\mx_t$ be the solution of SDE (\ref{stg1}).
Under Assumption 2, we have
\begin{equation}\label{au1}
\mathbb{E}[\|\mx_t\|_2^2]\leq C_{1,\tau,d} \exp(\tau^2t), \ \  \forall t \in [0,1],
\end{equation}
\begin{equation}\label{au2}
\begin{aligned}
& \mathbb{E}\left[\|\mx_{t_2}-\mx_{t_1}\|_2^2\right]\leq C_{2,\tau,d}((t_2-t_1)^2+(t_2-t_1)),\\
& \forall t_1, t_2 \in [0,1].
\end{aligned}
\end{equation}
\end{lemma}
\begin{proof}
By the definition of $\mx_t$ in (\ref{stg1}), we have
$
\|\mx_t\|_2\leq \int_{0}^t\tau\|D_1(\ell,\mx_{\ell})\|_2\mathrm{d}\ell+\sqrt{\tau}\|\mw_t\|_2.
$
Then,
\begin{align*}
\|\mx_t\|_2^2&\leq
2\tau^2\left(\int_{0}^t\|D_1(\ell,\mx_{\ell})\|_2\mathrm{d}\ell \right)^2+2\tau\|\mw_t\|_2^2\\
&\leq
2\tau^2t\int_{0}^t\|D_1(\ell,\mx_{\ell})\|_2^2\mathrm{d}\ell+2\tau\|\mw_t\|_2^2\\
&\leq
2\tau^2t\int_{0}^tC_1[\|\mx_{\ell}\|_2^2+1]\mathrm{d}\ell+2\tau\|\mw_t\|_2^2,
\end{align*}
where the first inequality holds due to the inequality $(a+b)^2\leq 2a^2+2b^2$, the last inequality holds by Assumption 2.
Thus,
\begin{align*}
\mathbb{E}[\|\mx_t\|_2^2]&\leq
2\tau^2t\int_{0}^tC_1(\mathbb{E}[\|\mx_{\ell}\|_2^2]+1)\mathrm{d}\ell+2\tau \mathbb{E}[\|\mw_t\|_2^2]\\
&\leq
2\tau^2C_1\int_{0}^t \mathbb{E}[\|\mx_{\ell}\|_2^2]\mathrm{d}\ell+(2\tau^2C_1+2\tau d).
\end{align*}
Then, (\ref{au1}) follows from the above display and the Bellman-Gronwall inequality.

Again, by the definition of $\mx_t$  in (\ref{stg1}), we have
\begin{align*}
\|\mx_{t_2}-\mx_{t_1}\|_2\leq \int_{t_1}^{t_2}\tau\|D_1(\mx_{\ell},\ell)\|_2\mathrm{d}\ell+\sqrt{\tau}\|\mw_{t_2}-\mw_{t_1}\|_2,
\end{align*}
Then,
\begin{align*}
&\|\mx_{t_2}-\mx_{t_1}\|_2^2\\
\leq&
2\tau^2\left(\int_{t_1}^{t_2}\|D_1(\mx_{\ell},\ell)\|_2\mathrm{d}\ell\right)^2+2\tau\|\mw_{t_2}-\mw_{t_1}\|_2^2\\
\leq&
2\tau^2(t_2-t_1)\int_{t_1}^{t_2}\|D_1(\mx_{\ell},\ell)\|_2^2\mathrm{d}\ell+2\tau\|\mw_{t_2}-\mw_{t_1}\|_2^2\\
\leq&
2\tau^2(t_2-t_1)\int_{t_1}^{t_2}C_1[\|\mx_{\ell}\|_2^2+1]\mathrm{d}\ell+2\tau\|\mw_{t_2}-\mw_{t_1}\|_2^2,
\end{align*}
where the last inequality holds by by Assumption 2.
Taking expectations on both sides and using  (\ref{au1}), we get (\ref{au2}).
\end{proof}

\section{Hyperparameter Settings}
For the two-dimensional toy example, we set batch size to be $1000$, and use the Adam optimizer \cite{kingma2014adam} for both the score estimator and the density ratio estimator. We use learning rate $lr=0.0001$ and exponential decay rates $betas=(0.5, 0.999)$ for the moment estimates when training the score estimator, and use $lr=0.001$, $betas=(0.5, 0.999)$ and L2 penalty $weight\_decay=0.1$ for the density ratio estimator. For the image datasets, the batch size is $128$ for both networks. We use $lr=0.0001$, $betas=(0.9, 0.999)$ and $eps=10^{-8}$ for the score estimator, and $lr=10^{-5}$, $betas=(0.5, 0.999)$ and $weight\_decay=1.0$ for the density ratio estimator.

\section{Network Architectures}
The score estimator $\hat{\ms}_{\theta}(\cdot, \cdot)$ and the density ratio estimator $\hat{f}(\cdot) = \exp(\hat{r}_{\phi}(\cdot))$ are parameterized with fully connected networks for the 2D example. The details are listed in Tables \ref{score-2d} and \ref{dse-2d}.
\begin{table}[H]
\caption{$\hat{\ms}_{\theta}$ for 2D example. $\mathbf{T}$ represents the sinusoidal embeddings \cite{vaswani2017attention} of time $t$.}
\label{score-2d}
\vskip 0.15in
\begin{center}
\begin{small}
\begin{sc}
\begin{tabular}{lcc}
\toprule
Layer & Detail & Output Size\\
\midrule
Fully Connected &Linear&$256$\\
 &Add $\text{Linear}_1(\mathbf{T})$&$256$\\
 & RELU &$256$\\
\midrule
Fully Connected &Linear&$512$\\
 &Add $\text{Linear}_2(\mathbf{T})$&$512$\\
 & RELU &$512$\\
\midrule
Fully Connected &Linear&2\\
\bottomrule
\end{tabular}
\end{sc}
\end{small}
\end{center}
\vskip -0.1in
\end{table}
\begin{table}[H]
\caption{$\hat{r}_{\phi}$ for 2D example.}
\label{dse-2d}
\vskip 0.15in
\begin{center}
\begin{small}
\begin{sc}
\begin{tabular}{lcc}
\toprule
Layer & Detail & Output Size\\
\midrule
Fully Connected &Linear&$256$\\
 & RELU &$256$\\
\midrule
Fully Connected &Linear&$512$\\
 & RELU &$512$\\
\midrule
Fully Connected &Linear&1\\
\bottomrule
\end{tabular}
\end{sc}
\end{small}
\end{center}
\vskip -0.1in
\end{table}
For image datasets, we parameterize the density ratio estimator with a residual network. The structure of $\hat{r}_{\phi}$ is list in Table \ref{dse}. Our choice of network architecture for $\hat{\ms}_\theta$ follows the implementation of the noise predictor $\meps_\theta$ in \cite{song2021denoising} which is a U-Net \cite{ronneberger2015u} based on a Wide ResNet \cite{zagoruyko2016wide}.
\begin{table}[H]
\caption{$\hat{r}_{\phi}$ with $32\times32\times3$ resolution.}
\label{dse}
\vskip 0.15in
\begin{center}
\begin{small}
\begin{sc}
\begin{tabular}{lcc}
\toprule
Layer & Detail & Output Size\\
\midrule
Conv Block &Conv $5\times5$&$32\times32\times128$\\
 & RELU &$32\times32\times128$\\
\midrule
Residual Block &Conv $5\times5$&$32\times32\times128$\\
 & RELU &$32\times32\times128$\\
\midrule
Residual Block &Conv $3\times3$&$32\times32\times128$\\
 & RELU &$32\times32\times128$\\
\midrule
Residual Block &Conv $3\times3$&$32\times32\times128$\\
 & RELU &$32\times32\times128$\\
\midrule
Conv Block &Conv $3\times3$&$32\times32\times128$\\
 & RELU &$32\times32\times128$\\
\midrule
Fully Connected &Linear&1\\
\bottomrule
\end{tabular}
\end{sc}
\end{small}
\end{center}
\vskip -0.1in
\end{table}

\section{More Implementation Details}

When training $\hat{f}(\mx)$, we substract an estimated image mean $\bar{\mx}$ from samples in $p_{\mathrm{data}}$ to center the data distributions at the origin. The data pre-processing is slightly different when training $\hat{\ms}_\theta(\mx)$, where the samples $\mx$ from $p_{\mathrm{data}}$ are only rescaled to $[-0.5, 0.5]$. We match the output $\hat{\ms}_\theta(\mx+\mz,\sigma)$ with $\frac{\mz}{\tilde{\sigma}^2}$ instead of $-\frac{\mz}{\tilde{\sigma}^2}$ in the denoising score matching objective. To make our algorithm be correctly implemented, we shift the input by adding $\bar{\mx}-0.5$ when using $\hat{\ms}_\theta(\mx)$, and adjust the sign of the output accordingly.

For image generation, there exist very small noises in the generated samples. To eliminate the negative effects induced by noises, we run one additional denoising step after stage 2, by repeating the last step without injecting any noise:
$$\mx_{N_2} = \mx_{N_2} + \frac{\sigma_0^2}{N_2}\mb(\mx_{N_2}),\quad \mb(\cdot)=\hat{\ms}_\theta(\cdot, \sqrt{\frac{1}{N_2}}\sigma_0).$$

We run all the experiments on CIFAR-10 and CelebA with one Tesla V100 GPU.

\section{Additional Experiment Results}
Here we first list the quantitive results with $\sigma^2\in\{0.5, 2.0, 5.0\}$, where results with $\sigma^2=1.0$ are already presented in the paper. We compare the results with different $\tau$ values starting $\tau_{\min}=\sigma^2$. The results are presented in Tables \ref{fid-05-tau}, \ref{fid-2-tau} and \ref{fid-5-tau}.
\begin{table}[H]
\caption{FID and Inception Score on CIFAR-10 with $\sigma^2=0.5$.}
\label{fid-05-tau}
\vskip 0.15in
\begin{center}
\begin{small}
\begin{sc}
\begin{tabular}{lcccccccc}
\toprule
$\tau$ & $0.5$ & $1.0$ & $1.5$ & $2.0$\\
\midrule
FID & 46.59 & 19.57 & {\bf 18.73} & 20.86\\
IS & 5.92 & 7.83 & {\bf 8.13} & 8.09\\
\bottomrule
\toprule
$\tau$ & $2.5$ & $3.0$ & $3.5$\\
\midrule
FID & 21.28 & 21.03 & 20.40\\
IS & 8.05 & 7.98 & 8.00\\
\bottomrule
\end{tabular}
\end{sc}
\end{small}
\end{center}
\vskip -0.1in
\end{table}
\begin{table}[H]
\caption{FID and Inception Score on CIFAR-10 with $\sigma^2=2.0$.}
\label{fid-2-tau}
\vskip 0.15in
\begin{center}
\begin{small}
\begin{sc}
\begin{tabular}{lcccccccc}
\toprule
$\tau$ & $2.0$ & $2.5$ & $3.0$ & $3.5$\\
\midrule
FID & 28.92 & 22.37 & 14.52 & 12.45\\
IS & 7.06 & 7.50 & 7.97 & {\bf 7.98}\\
\bottomrule
\toprule
$\tau$ & $4.0$ & $4.5$ & $5.0$\\
\midrule
FID & {\bf 12.27} & 12.58 & 12.87\\
IS & 7.91 & 7.86 & 7.81\\
\bottomrule
\end{tabular}
\end{sc}
\end{small}
\end{center}
\vskip -0.1in
\end{table}
\begin{table}[H]
\caption{FID and Inception Score on CIFAR-10 with $\sigma^2=5.0$.}
\label{fid-5-tau}
\vskip 0.15in
\begin{center}
\begin{small}
\begin{sc}
\begin{tabular}{lcccccccc}
\toprule
$\tau$ & $5.0$ & $5.5$ & $6.0$ & $6.5$\\
\midrule
FID & 17.80 & 17.52 & 18.24 & 16.46\\
IS & 7.67 & {\bf 7.68} & 7.66 & {\bf 7.68}\\
\bottomrule
\toprule
$\tau$ & $7.0$ & $7.5$ & $8.0$\\
\midrule
FID & 15.71 & 15.45 & {\bf 15.41}\\
IS & 7.64 & 7.62 & 7.59\\
\bottomrule
\end{tabular}
\end{sc}
\end{small}
\end{center}
\vskip -0.1in
\end{table}

\bibliography{paper}
\bibliographystyle{icml2021}

%
%
%

\end{document}